\pdfoutput=1
\documentclass[twoside,11pt]{article}
\usepackage[preprint]{jmlr2e}
\usepackage{amsmath,mathtools}
\usepackage{microtype}
\usepackage{enumitem}
\usepackage{booktabs}
\usepackage{array}
\usepackage{float}
\usepackage{afterpage}
\usepackage{xcolor}
\usepackage{tikz}
\usetikzlibrary{arrows.meta,calc,decorations.pathreplacing}
\usepackage{adjustbox}
\usepackage{lastpage}
\hypersetup{
  colorlinks=true,
  linkcolor=blue,
  citecolor=blue,
  urlcolor=blue,
  pdftitle={The Privacy Price of Tail-Risk Learning: Effective Tail Sample Size in Differentially Private CVaR Optimization},
  pdfauthor={El Mustapha Mansouri},
  pdfsubject={Differential privacy and tail-risk learning: minimax rates for private CVaR optimization},
  pdfkeywords={differential privacy; conditional value-at-risk; CVaR; tail-risk learning; private machine learning; private CVaR optimization; minimax rates; sample complexity; effective sample size; private stochastic convex optimization; distributionally robust optimization; robust learning; finite classes; convex Lipschitz learning}
}

\numberwithin{theorem}{section}
\newtheorem{assumption}[theorem]{Assumption}

\newcommand{\R}{\mathbb{R}}
\newcommand{\E}{\mathbb{E}}
\newcommand{\Pp}{\mathbb{P}}
\newcommand{\1}{\mathbf{1}}
\newcommand{\cvar}{\operatorname{CVaR}}

\newcommand{\argmin}{\operatorname*{argmin}}

\newcommand{\DP}{\mathsf{DP}}
\newcommand{\SCO}{\mathsf{SCO}}
\newcommand{\CVaR}{\mathsf{CVaR}}

\title{The Privacy Price of Tail-Risk Learning: Effective Tail Sample Size in Differentially Private CVaR Optimization}
\author{\name El Mustapha Mansouri \\
\addr School of Engineering, Institute of Science Tokyo, Yokohama, Kanagawa 226-8501, Japan}

\ShortHeadings{The Privacy Price of Tail-Risk Learning}{Mansouri}
\firstpageno{1}

\begin{document}
\maketitle

\begin{abstract}%
Differential privacy changes the effective sample size governing CVaR learning.
For tail mass $\tau$, the privacy-relevant sample size is not $n$, but $n\tau$;
equivalently, the effective private tail sample size is $\varepsilon n\tau$.
Private CVaR excess risk decomposes into ordinary tail-risk statistical error
and a privacy price. This decomposition is complete for scalar estimation and
finite classes: scalar estimation has rate
$\Theta\!\bigl(B\min\{1,(n\tau)^{-1/2}+(\varepsilon n\tau)^{-1}\}\bigr)$, and
finite classes of size $M$ have rate
$\Theta\!\bigl(B\min\{1,\sqrt{\log(2M)/(n\tau)}+\log(2M)/(\varepsilon n\tau)\}\bigr)$.
These complete rates hold under pure DP, and their lower bounds extend to
approximate DP in the stated small-$\delta$ regimes. For convex Lipschitz
learning, modular upper and lower reductions show that the CVaR-specific privacy
term necessarily scales as $1/(\varepsilon n\tau)$, with dimension dependence
inherited from private stochastic convex optimization. Together, these results
identify ordinary private learning on $\Theta(n\tau)$ informative tail records
as the canonical hard subproblem inside private CVaR learning.
\end{abstract}


\afterpage{%
\begin{figure}[H]
\centering
\begin{adjustbox}{max width=\textwidth}
\begin{tikzpicture}[
  x=1cm,y=1cm,
  font=\sffamily\small,
  >=Latex,
  dot/.style={circle,inner sep=0pt,minimum size=3.1pt},
  taildot/.style={circle,inner sep=0pt,minimum size=4.0pt},
  channel/.style={->,line width=1.05pt,draw=black!58,shorten >=2pt,shorten <=2pt},
  excessbox/.style={rounded corners=6pt,draw=black!18,line width=.55pt,fill=white,
    minimum width=3.25cm,minimum height=1.48cm,align=center,inner sep=6pt}
]
\definecolor{tailorange}{RGB}{203,92,47}
\definecolor{signalgray}{RGB}{190,194,199}

\begin{scope}[shift={(0,1.56)}]
  \foreach \idx in {0,...,17} {
    \pgfmathsetmacro{\x}{0.17*\idx}
    \node[dot,fill=signalgray!70] at (\x,0) {};
  }
  \foreach \idx in {18,...,23} {
    \pgfmathsetmacro{\x}{0.17*\idx}
    \node[taildot,fill=tailorange] at (\x,0) {};
  }
  \node[font=\sffamily\Large] at (3.48,-.58) {$n\tau$};
\end{scope}

\draw[channel] (4.30,1.56) .. controls (5.22,2.46) and (6.72,2.46) .. (8.20,1.79);
\node[font=\sffamily\Large,fill=white,inner sep=2pt] at (5.86,2.62)
  {$\displaystyle \frac{1}{\sqrt{n\tau}}$};

\draw[channel] (4.30,1.56) .. controls (5.22,.66) and (6.72,.66) .. (8.20,1.33);
\node[font=\sffamily\Large,fill=white,inner sep=2pt] at (5.86,.46)
  {$\displaystyle \frac{1}{\varepsilon n\tau}$};

\node[excessbox] at (10.22,1.56) {
  {\bfseries Private Excess CVaR}\\[6pt]
  $\displaystyle
  \frac{1}{\sqrt{n\tau}}
  +
  \frac{1}{\varepsilon n\tau}$
};
\end{tikzpicture}
\end{adjustbox}
\caption{\textbf{Error decomposition in private CVaR learning.} CVaR concentrates learning on the worst $\tau$-fraction of the sample, leaving only $n\tau$ informative tail observations. Ordinary tail-risk estimation contributes the statistical term $1/\sqrt{n\tau}$, while differential privacy adds a separate price of order $1/(\varepsilon n\tau)$. Thus privacy acts on the effective private tail sample size $\varepsilon n\tau$, not on the full sample size $n$.}
\label{fig:private-cvar-error-decomposition}
\end{figure}
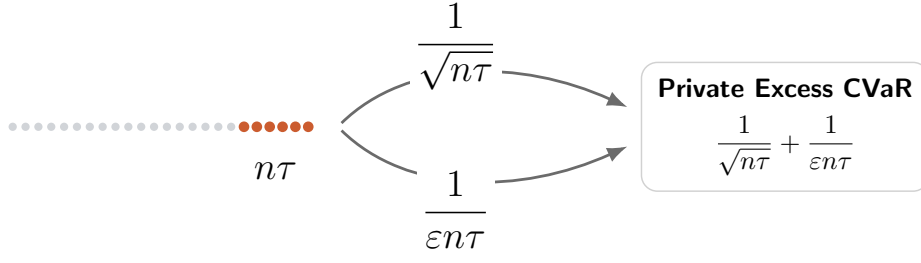%
}

\paragraph{Asymptotic notation.}
For nonnegative quantities $U$ and $V$ depending on the problem parameters, $U=O(V)$ means that $U\le CV$ for a universal numerical constant $C<\infty$ over the stated parameter regime, and $U=\Omega(V)$ means that $U\ge cV$ for a universal numerical constant $c>0$. The notation $U=\Theta(V)$ is used when both bounds hold; $U\asymp V$ has the same meaning. The symbols $\lesssim$ and $\gtrsim$ denote the corresponding one-sided inequalities up to universal constants. Tilde notation, such as $\widetilde O(\cdot)$ and $\widetilde\Omega(\cdot)$, hides only polylogarithmic factors in the dimensionless parameters in force, such as $n,d,M,1/\tau,1/\varepsilon$, and $1/\delta$; it never hides additional polynomial dependence on these parameters or any dependence on scale parameters such as $B,G$, and $D$.

\section{Introduction}

Empirical risk minimization optimizes average performance. In heterogeneous populations, however, a small average loss can coexist with large loss on rare or poorly represented subpopulations; this phenomenon motivates robust and distributionally robust learning objectives for minority-group performance \citep{HashimotoSrivastavaNamkoongLiang2018}. A standard mathematical response is to replace expectation by a tail-risk functional. For a nonnegative loss $L$ and tail mass $\tau\in(0,1]$, define
\[
\rho_\tau(L)
:=
\cvar_{1-\tau}(L)
=
\inf_{\eta\in\R}\left\{\eta+\frac1\tau\E[(L-\eta)_+]\right\}.
\]
This is the average loss in the worst $\tau$ fraction of the distribution, with the variational formulation due to \citet{RockafellarUryasev2000}. The same functional is known as expected shortfall or average value-at-risk, and it is one of the canonical law-invariant coherent risk measures \citep{ArtznerDelbaenEberHeath1999,AcerbiTasche2002,Shapiro2013}.

The central question is what happens when the learner must also satisfy differential privacy. The question is not merely algorithmic. The samples that determine CVaR are precisely the high-loss tail samples. In applications, those records often correspond to rare, atypical, high-impact, or minority cases. Thus the privacy and robustness requirements collide at the same observations.

The privacy model is record-level central differential privacy \citep{DworkMcSherryNissimSmith2006,DworkRoth2014}. Given iid data $S=(Z_1,\ldots,Z_n)$ from $P$, a hypothesis class $\mathcal F$, and bounded losses $\ell(f;Z)\in[0,B]$, the population objective is
\[
\rho_{\tau,P}(f)
=
\inf_{\eta\in[0,B]}\left\{\eta+\frac1\tau\E_P[(\ell(f;Z)-\eta)_+]\right\}.
\]
The central quantity is the private minimax excess CVaR risk
\[
\mathfrak R^{\varepsilon,\delta}_{n,\tau}(\mathcal F)
:=
\inf_{\mathcal A\in(\varepsilon,\delta)\text{-}\DP}
\sup_P
\E\left[
\rho_{\tau,P}(\mathcal A(S))-
\inf_{f\in\mathcal F}\rho_{\tau,P}(f)
\right].
\]
The expectation is over both the sample and the learner's internal randomness.

\paragraph{Main message.}
Private CVaR learning decomposes into ordinary tail-risk statistics and a privacy price. This decomposition is closed completely for scalar estimation and finite classes. In convex Lipschitz learning, the same CVaR-specific privacy price persists through modular reductions to and from private stochastic convex optimization. Thus the results identify the inverse-tail-mass factor $1/\tau$ as the intrinsic privacy price of tail-risk learning:
\[
\text{private CVaR excess}
\quad=\quad
\text{ordinary CVaR statistical error}
\quad+\quad
\text{privacy price}.
\]
The new privacy phenomenon is entirely in the second term. The correct private sample size is not $n$. It is $n\tau$. Equivalently, the privacy price scales as
\[
\frac{\text{complexity}}{\varepsilon n\tau},
\]
not as $\text{complexity}/(\varepsilon n)$. This is the privacy price of tail-risk learning.
In the CVaR dual representation, the relevant record leverage is exactly $1/\tau$, which is why the private sample size becomes $n\tau$.
Figure~\ref{fig:private-cvar-error-decomposition} summarizes this decomposition.

\paragraph{Why the factor $1/\tau$ is unavoidable.}
The empirical CVaR objective is
\[
\widehat\rho_{\tau,S}(f)
=
\inf_{\eta\in[0,B]}\left\{\eta+\frac1{n\tau}\sum_{i=1}^n(\ell(f;Z_i)-\eta)_+\right\}.
\]
Changing one record changes the lifted empirical objective by at most $B/(n\tau)$, and the minimized empirical CVaR by $B\min\{1,1/(n\tau)\}$. This is exactly the sensitivity to which differential privacy calibrates noise. The lower bounds show that this sensitivity calculation is not merely a proof artifact: for scalar estimation, for finite-class learning, and for convex learning via tail embedding, an inverse-tail-mass privacy penalty is unavoidable.
Figure~\ref{fig:privacy-tail-frontier} traces how the tail focus, one-record sensitivity, and private tail-sample scale fit together.

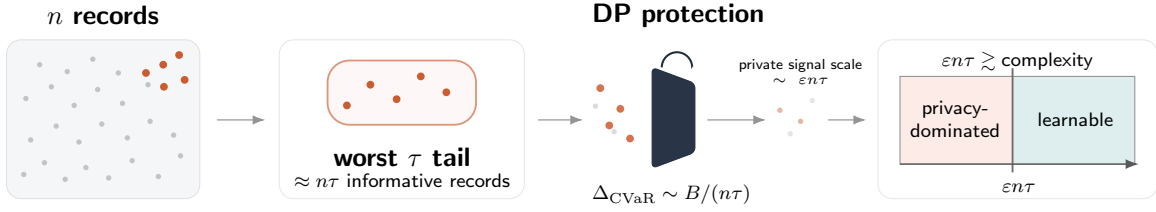
\begin{figure}[t!]
\centering
\begin{adjustbox}{max width=\textwidth}
\begin{tikzpicture}[
  x=1cm,y=1cm,
  font=\sffamily\small,
  >=Latex,
  dot/.style={circle,inner sep=0pt,minimum size=2.2pt},
  taildot/.style={circle,inner sep=0pt,minimum size=3.2pt},
  arrow/.style={->,line width=0.5pt,draw=black!42}
]
\definecolor{tailorange}{RGB}{203,92,47}
\definecolor{paletail}{RGB}{252,235,231}
\definecolor{learnblue}{RGB}{222,238,238}
\definecolor{navygate}{RGB}{32,45,62}
\definecolor{softgray}{RGB}{245,246,248}
\definecolor{signalgray}{RGB}{186,190,196}

\begin{scope}[shift={(0,0)}]
  \draw[rounded corners=6pt,fill=softgray,draw=black!12] (0,0.28) rectangle (2.8,2.58);
  \node[font=\sffamily\bfseries] at (1.4,2.94) {$n$ records};
  \foreach \x/\y in {
    .25/.63,.55/.80,.85/.56,1.15/.73,1.48/.58,1.82/.83,2.20/.60,2.52/.90,
    .35/1.13,.72/1.33,1.05/1.10,1.38/1.36,1.75/1.16,2.08/1.40,2.42/1.20,
    .28/1.73,.65/1.90,.98/1.63,1.32/1.86,1.68/1.66,2.05/1.93,
    .48/2.22,.88/2.12,1.25/2.30,1.62/2.10}
    \node[dot,fill=black!23] at (\x,\y) {};
  \foreach \x/\y in {2.02/2.10,2.27/2.22,2.50/2.36,2.28/1.90,2.58/2.00}
    \node[taildot,fill=tailorange] at (\x,\y) {};
\end{scope}

\draw[arrow] (3.05,1.42) -- (3.75,1.42);

\begin{scope}[shift={(3.95,0)}]
  \draw[rounded corners=6pt,fill=white,draw=black!10] (0,0.28) rectangle (3.55,2.58);
  \draw[rounded corners=9pt,fill=tailorange!4,draw=tailorange!65,line width=0.75pt]
    (.70,1.35) rectangle (2.85,2.28);
  \foreach \x/\y in {.98/1.63,1.32/1.93,1.70/1.72,2.05/2.05,2.42/1.82}
    \node[taildot,fill=tailorange] at (\x,\y) {};
  \node[font=\sffamily\bfseries,align=center] at (1.78,.88) {worst $\tau$ tail};
  \node[font=\sffamily\scriptsize,align=center] at (1.78,.54) {$\approx n\tau$ informative records};
\end{scope}

\draw[arrow] (7.70,1.42) -- (8.35,1.42);

\begin{scope}[shift={(8.55,0)}]
  \node[font=\sffamily\bfseries] at (1.18,2.94) {DP protection};
  \foreach \x/\y in {-.02/1.62,.25/1.25}
    \node[dot,fill=signalgray!55] at (\x,\y) {};
  \foreach \x/\y in {.05/1.88,.34/1.58,.18/1.34,.48/1.15}
    \node[taildot,fill=tailorange!85] at (\x,\y) {};
  \draw[fill=navygate!96,draw=navygate,rounded corners=2pt]
    (.80,.72) -- (1.42,.95) -- (1.42,1.95) -- (.80,2.18) -- cycle;
  \draw[draw=navygate,line width=0.85pt] (.94,2.18) arc[start angle=180,end angle=0,radius=.22] -- (1.38,2.18);
  \node[align=center,font=\sffamily\scriptsize] at (1.08,.34) {$\Delta_{\mathrm{CVaR}}\sim B/(n\tau)$};
  \draw[arrow] (1.60,1.42) -- (2.45,1.42);
  \foreach \x/\y in {2.65/1.56,2.95/1.40}
    \node[dot,fill=tailorange!45] at (\x,\y) {};
  \foreach \x/\y in {2.80/1.22,3.12/1.70}
    \node[dot,fill=signalgray!40] at (\x,\y) {};
  \node[align=center,font=\sffamily\tiny,text width=2.00cm] at (2.95,2.12) {private signal scale\\[-1pt]$\sim \varepsilon n\tau$};
\end{scope}

\draw[arrow] (11.90,1.42) -- (12.45,1.42);

\begin{scope}[shift={(12.62,0)}]
  \draw[rounded corners=6pt,draw=black!10,fill=white] (0,0.28) rectangle (4.05,2.58);
  \fill[paletail] (.32,.78) rectangle (1.95,2.05);
  \fill[learnblue] (1.95,.78) rectangle (3.72,2.05);
  \draw[black!42,line width=.45pt] (.32,.78) rectangle (3.72,2.05);
  \draw[black!58,line width=.75pt] (1.95,.66) -- (1.95,2.17);
  \draw[->,line width=.55pt,draw=black!65] (.32,.78) -- (3.78,.78);
  \node[font=\sffamily\scriptsize] at (2.02,2.27) {$\varepsilon n\tau \gtrsim \text{complexity}$};
  \node[align=center,font=\sffamily\scriptsize] at (1.12,1.43) {privacy-\\dominated};
  \node[align=center,font=\sffamily\scriptsize] at (2.86,1.43) {learnable};
  \node[font=\sffamily\scriptsize] at (2.04,.43) {$\varepsilon n\tau$};
\end{scope}
\end{tikzpicture}
\end{adjustbox}
\caption{\textbf{The Privacy--Tail-Risk Frontier in Private CVaR Learning.} CVaR concentrates the objective on the worst $\tau$-fraction of the loss distribution, so only about $n\tau$ records carry tail information. Record-level differential privacy must protect each high-loss tail record, producing sensitivity of order $B/(n\tau)$. Consequently, the privacy-relevant sample size is $\varepsilon n\tau$, not $\varepsilon n$. Nontrivial private tail-risk learning requires this effective private tail sample size to dominate the relevant model complexity.}
\label{fig:privacy-tail-frontier}
\end{figure}

\subsection{Contributions}

The main theoretical contributions are as follows.

\begin{enumerate}[leftmargin=2em]
\item A sharp sensitivity lemma is proved for empirical CVaR and its Rockafellar-Uryasev lifted objective: for bounded losses, minimized empirical CVaR has exact one-record sensitivity $B\min\{1,1/(n\tau)\}$, while the lifted objective at fixed $(f,\eta)$ has sensitivity $B/(n\tau)$.

\item The complete scalar minimax rate is established. Estimating $\rho_\tau(Z)$ privately for $Z\in[0,B]$ has rate
\[
\Theta\left(
B\min\left\{
1,
\frac1{\sqrt{n\tau}}
+
\frac1{\varepsilon n\tau}
\right\}
\right),
\]
under pure differential privacy, with the lower bound extending to approximate differential privacy whenever $\delta$ is small enough for group privacy to be informative.

\item The complete finite-class minimax rate is established under pure DP and under approximate DP in the stated small-$\delta$ regime. For $|\mathcal F|=M$, ordinary CVaR uniform convergence contributes $B\sqrt{\log(2M)/(n\tau)}$, while privacy contributes $B\log(2M)/(\varepsilon n\tau)$. Together they yield
\[
\Theta\left(
B\min\left\{
1,
\sqrt{\frac{\log(2M)}{n\tau}}
+
\frac{\log(2M)}{\varepsilon n\tau}
\right\}
\right),
\]
exhibiting the effective sample size $n\tau$ and the effective private sample size $\varepsilon n\tau$ side by side.

\item A modular approximate-DP convex upper bound is proved using a rescaled lifted convex problem and the standard Euclidean private-SCO population-risk interface. For $G$-Lipschitz convex losses over a Euclidean set of diameter $D$, the private term is
\[
\widetilde O\left(\frac{(GD+B)\sqrt{d\log(1/\delta)}}{\varepsilon n\tau}\right).
\]

\item A tail-sample transfer theorem identifies the canonical hard subproblem inside CVaR learning: ordinary expected-risk learning with $m$ samples is embedded exactly inside CVaR learning with $n$ samples and $m\asymp n\tau$ informative tail observations. Combining this embedding with the Euclidean approximate-DP SCO lower interface yields dimension-dependent CVaR lower bounds, up to imported SCO logarithmic factors, of order
\[
\widetilde\Omega\left(\min\{B,GD\}\min\left\{1,\frac1{\sqrt{n\tau}}+\frac{\sqrt d}{\varepsilon n\tau}\right\}\right)
\]
for a natural bounded convex Lipschitz CVaR subclass.
\end{enumerate}

The resulting phase transition is simple: nontrivial private tail-risk learning requires
\[
\varepsilon n\tau
\gg
\text{model complexity}.
\]
In Euclidean convex learning this complexity is of order $\sqrt d$, up to logarithmic factors.

\paragraph{Pure and approximate DP.}
The minimax equalities are stated for pure DP. For approximate DP, the same lower-bound mechanisms extend in the small-$\delta$ regimes where group privacy remains informative; these regimes are stated explicitly rather than hidden in asymptotic notation.
The convex Lipschitz results are modular transfer results in the standard approximate-DP Euclidean SCO regime; the scalar and finite-class minimax equalities are the complete pure-DP characterizations.

\begin{table}[tbp]
\centering
\footnotesize
\setlength{\tabcolsep}{5pt}
\renewcommand{\arraystretch}{1.16}
\caption{Summary of the main guarantees. The scalar, finite-class, and convex rows give theorem-level minimax evidence for the $1/\tau$ privacy frontier. The finite-class row is the closed model where the ordinary statistical term and the privacy price are both characterized. The coherent-envelope row records only the general sensitivity mechanism; the minimax theory is proved for CVaR.}
\begin{tabular}{>{\raggedright\arraybackslash}p{0.18\linewidth}
                >{\raggedright\arraybackslash}p{0.53\linewidth}
                >{\raggedright\arraybackslash}p{0.20\linewidth}}
\toprule
Setting & Main guarantee & Takeaway \\
\midrule
Scalar CVaR estimation &
Complete rate under pure DP; approximate-DP lower bound under small $\delta$:
$\Theta(B\min\{1,1/\sqrt{n\tau}+1/(\varepsilon n\tau)\})$. &
Closed scalar decomposition. \\
\addlinespace
Finite class, $|\mathcal F|=M$ &
Complete rate under pure DP; approximate DP under small $\delta$:
$\Theta(B\min\{1,\sqrt{\log(2M)/(n\tau)}+\log(2M)/(\varepsilon n\tau)\})$. &
Ordinary tail sample size $n\tau$ plus private tail sample size $\varepsilon n\tau$. \\
\addlinespace
Convex Lipschitz learning &
Upper privacy term: $\widetilde O((GD+B)\sqrt d/(\varepsilon n\tau))$. Tail embedding gives
$\widetilde\Omega(R_0\min\{1,1/\sqrt{n\tau}+\sqrt d/(\varepsilon n\tau)\})$. &
Privacy dependence sharp in $(\varepsilon,n,\tau,d)$; scale-sharp under $B\asymp GD$. \\
\addlinespace
Envelope-bounded coherent risks &
One-record empirical sensitivity is at most $B\min\{1,\kappa/n\}$. No general coherent-risk minimax theorem is claimed. &
Sensitivity extension only; CVaR is the resolved case $\kappa=1/\tau$. \\
\bottomrule
\end{tabular}
\end{table}

\subsection{Related work}

\paragraph{CVaR and coherent risk.}
Coherent risk measures were axiomatized by \citet{ArtznerDelbaenEberHeath1999}. The variational representation used throughout the analysis is due to \citet{RockafellarUryasev2000}, who also developed CVaR for general loss distributions \citep{RockafellarUryasev2002}. Expected shortfall/CVaR is coherent under the standard definitions studied by \citet{AcerbiTasche2002}. Kusuoka-type representations show that law-invariant coherent risk measures can be represented using average value-at-risk functionals \citep{Shapiro2013}. \citet{RuszczynskiShapiro2006} developed optimization theory for convex risk functions.

\paragraph{CVaR and robust learning.}
CVaR-style and DRO-style losses have been used to improve worst-tail or worst-subpopulation performance. CVaR-based statistical learning formulations go back at least to \citet{TakedaKanamori2009}. \citet{HashimotoSrivastavaNamkoongLiang2018} use distributionally robust objectives to mitigate minority-group underperformance without explicit demographic labels. \citet{DuchiHashimotoNamkoong2023} propose robust losses for latent covariate mixtures and explicitly motivate the method by failures of average-loss minimization on heterogeneous subpopulations. PAC-Bayesian and statistical learning analyses of CVaR have also appeared \citep{Mhammedi2020}. A recent preprint develops sharp nonprivate generalization and robustness results for CVaR-ERM under heavy-tailed and contaminated data \citep{Mulumudi2026}. These works do not characterize the differential privacy price of CVaR learning.

\paragraph{Differential privacy and private learning.}
Differential privacy was introduced through sensitivity-calibrated randomized mechanisms by \citet{DworkMcSherryNissimSmith2006}; for general background on differential privacy, group privacy, and post-processing, see \citet{DworkRoth2014}. The finite-class upper bound uses the exponential mechanism of \citet{McSherryTalwar2007}. Private convex ERM was studied by \citet{ChaudhuriMonteleoniSarwate2011} and later by \citet{BassilySmithThakurta2014}. \citet{BassilyFeldmanTalwarThakurta2019} established near-optimal population-risk rates for private stochastic convex optimization. This theory is used as a primitive after reducing CVaR learning to a lifted convex problem.

\paragraph{Nearby privacy-robustness work.}
Recent work studies differentially private worst-group risk minimization for known finite groups \citep{ZhouBassily2024}. CVaR is different: the high-loss tail is unknown, endogenous, and predictor-dependent. Differentially private quantile estimation and quantile-loss optimization are also close but distinct \citep{GillenwaterJosephKulesza2021,ChenChua2023}. The lower bounds include cases where the VaR threshold is known; the difficulty is private tail averaging, not merely private quantile estimation. Finally, a recent preprint on DP-DRO studies private algorithms for DRO formulations, including divergence-based models, but it does not settle the minimax population excess-risk frontier for standard unregularized CVaR-ERM \citep{XuEtAl2026}.

\begin{table}[H]
\centering
\footnotesize
\setlength{\tabcolsep}{4pt}
\renewcommand{\arraystretch}{1.12}
\caption{Relationship to nearby lines of work. The distinguishing feature here is private learning of an endogenous CVaR tail average at the effective tail sample size $n\tau$.}
\begin{tabular}{>{\raggedright\arraybackslash}p{0.19\linewidth}
                >{\raggedright\arraybackslash}p{0.19\linewidth}
                >{\centering\arraybackslash}p{0.10\linewidth}
                >{\centering\arraybackslash}p{0.15\linewidth}
                >{\raggedright\arraybackslash}p{0.26\linewidth}}
\toprule
Line of work & Objective & Privacy & Endogenous tail & Minimax $n\tau$ frontier \\
\midrule
Private SCO & Expectation & Yes & No & No tail-risk sample-size effect \\
Private quantiles & VaR/quantile & Yes & Partly & No private tail-average term \\
Worst-group risk & Known groups & Yes & No & Group sizes are fixed externally \\
DP-DRO & Robust objective & Yes & Varies & Not unregularized CVaR minimax \\
Private CVaR learning & CVaR tail average & Yes & Yes & Yes: scalar and finite-class complete, convex modular \\
\bottomrule
\end{tabular}
\end{table}

\section{Preliminaries}

\subsection{Differential privacy}

Datasets $S,S'\in\mathcal Z^n$ are neighboring, denoted $S\sim S'$, if they differ in at most one coordinate.

\begin{definition}[Differential privacy]
A randomized algorithm $\mathcal A:\mathcal Z^n\to\mathcal O$ is $(\varepsilon,\delta)$-differentially private if for all neighboring $S,S'$ and all measurable $E\subseteq\mathcal O$,
\[
\Pp(\mathcal A(S)\in E)
\le
 e^\varepsilon \Pp(\mathcal A(S')\in E)+\delta.
\]
If $\delta=0$, the algorithm is $\varepsilon$-differentially private.
\end{definition}

The following standard group-privacy consequence is used \citep{DworkRoth2014}.

\begin{lemma}[Group privacy]
If $\mathcal A$ is $(\varepsilon,\delta)$-DP and $S,S'$ differ in at most $k$ entries, then for all events $E$,
\[
\Pp(\mathcal A(S')\in E)
\le
 e^{k\varepsilon}\Pp(\mathcal A(S)\in E)
 +
 \delta\sum_{j=0}^{k-1}e^{j\varepsilon}.
\]
Equivalently,
\[
\Pp(\mathcal A(S)\in E)
\ge
 e^{-k\varepsilon}\Pp(\mathcal A(S')\in E)-k\delta.
\]
\end{lemma}

\begin{proof}
Iterate the one-step DP inequality along a path of $k$ neighboring datasets. The second inequality follows by rearranging and using $e^{(j-k)\varepsilon}\le1$.
\end{proof}

\subsection{CVaR}

For an integrable real random variable $X$ and $\tau\in(0,1]$, define
\[
\rho_\tau(X)
=
\inf_{\eta\in\R}\left\{\eta+\frac1\tau\E[(X-\eta)_+]\right\}.
\]
When $X$ is a loss, this is $\cvar_{1-\tau}(X)$, the average loss over the upper $\tau$ tail \citep{RockafellarUryasev2000,RockafellarUryasev2002}. If $X\in[0,B]$, the infimum may be restricted to $\eta\in[0,B]$.

A useful dual representation is
\[
\rho_\tau(X)
=
\sup\left\{\E[qX]:0\le q\le\frac1\tau,\;\E q=1\right\}.
\]
This representation makes clear that CVaR is an adversarially reweighted expectation whose density envelope is $1/\tau$ \citep{RockafellarUryasev2002,RuszczynskiShapiro2006}.

\subsection{Learning objective}

Let $\mathcal F$ be a class of predictors and let $\ell:\mathcal F\times\mathcal Z\to[0,B]$ be a loss. For $f\in\mathcal F$,
\[
\rho_{\tau,P}(f)
:=
\rho_\tau(\ell(f;Z)),\qquad Z\sim P.
\]
The empirical counterpart is
\[
\widehat\rho_{\tau,S}(f)
=
\inf_{\eta\in[0,B]}\left\{
\eta+\frac1{n\tau}\sum_{i=1}^n(\ell(f;Z_i)-\eta)_+
\right\}.
\]
For convex learning, write $f=w\in W\subset\R^d$ and define the lifted population and empirical objectives
\begin{align*}
\phi_P(w,\eta)
&=
\eta+\frac1\tau\E_P[(\ell(w;Z)-\eta)_+],\\
\widehat\phi_S(w,\eta)
&=
\eta+\frac1{n\tau}\sum_{i=1}^n(\ell(w;Z_i)-\eta)_+.
\end{align*}
Then
\[
\rho_{\tau,P}(w)=\inf_{\eta\in[0,B]}\phi_P(w,\eta),
\qquad
\inf_w\rho_{\tau,P}(w)=\inf_{w,\eta}\phi_P(w,\eta).
\]

\section{Sensitivity and scalar CVaR estimation}

\subsection{Sensitivity of empirical CVaR}

\begin{lemma}[Exact one-record sensitivity]
Assume $\ell(f;z)\in[0,B]$ for all $f,z$. If $S\sim S'$, then
\[
\sup_f\left|\widehat\rho_{\tau,S}(f)-\widehat\rho_{\tau,S'}(f)\right|
\le
B\min\left\{1,\frac1{n\tau}\right\}.
\]
This bound is tight for the minimized empirical CVaR functional. Moreover, the lifted objective at fixed $(f,\eta)$ has sensitivity
\[
\sup_{f,\eta}\left|\widehat\phi_S(f,\eta)-\widehat\phi_{S'}(f,\eta)\right|
\le
\frac{B}{n\tau}.
\]
\end{lemma}

\begin{proof}
For a vector $x=(x_1,\ldots,x_n)\in[0,B]^n$, empirical CVaR has the dual form
\[
\widehat\rho_\tau(x)
=
\sup\left\{
\frac1n\sum_{i=1}^n q_i x_i:
0\le q_i\le \frac1\tau,\ 
\frac1n\sum_{i=1}^n q_i=1
\right\}.
\]
If $x,x'$ differ only in coordinate $j$, then every feasible $q$ satisfies
\[
\left|
\frac1n\sum_iq_ix_i-\frac1n\sum_iq_ix_i'
\right|
=
\frac{q_j}{n}|x_j-x_j'|
\le
\frac{B}{n\tau}.
\]
Taking suprema over the same feasible set gives
\[
|\widehat\rho_\tau(x)-\widehat\rho_\tau(x')|
\le
\frac{B}{n\tau}.
\]
Since both empirical CVaR values lie in $[0,B]$, also
\[
|\widehat\rho_\tau(x)-\widehat\rho_\tau(x')|\le B.
\]
The exact upper bound follows. It is attained by $x=(0,\ldots,0)$ and $x'=(B,0,\ldots,0)$, for which
\[
\widehat\rho_\tau(x)=0,
\qquad
\widehat\rho_\tau(x')
=
\frac{B}{n}\min\left\{n,\frac1\tau\right\}
=
B\min\left\{1,\frac1{n\tau}\right\}.
\]

For fixed $(f,\eta)$ and neighboring samples $S,S'$ differing only in the last coordinate,
\[
\left|\widehat\phi_S(f,\eta)-\widehat\phi_{S'}(f,\eta)\right|
\le
\frac1{n\tau}\left|(\ell(f;Z_n)-\eta)_+-(\ell(f;Z'_n)-\eta)_+\right|
\le
\frac{B}{n\tau},
\]
because $\eta\in[0,B]$ and both positive parts lie in $[0,B]$.
\end{proof}

This lemma identifies both scales. Minimized empirical CVaR has the exact capped sensitivity above, while algorithms acting on the lifted objective still face the $B/(n\tau)$ scale pointwise in $(f,\eta)$.

\subsection{Scalar lower bound}

The scalar problem is a private statistical estimation problem in the sense of optimal-rate private estimation \citep{Smith2011}. The next result proves that this scale is information-theoretically unavoidable even when there is no learning problem and the VaR threshold is known.

\begin{theorem}[Private scalar CVaR estimation lower bound]
\label{thm:scalar-privacy-lower}
There are universal constants $c,c_1,c_2>0$ such that the following holds. Let $\tau\in(0,1]$, $B>0$, and $\varepsilon\in(0,1]$. For any $\varepsilon$-DP estimator $\widehat r:\mathcal Z^n\to\R$,
\[
\sup_{P:Z\in[0,B]}
\E_{S\sim P^n}\left|\widehat r(S)-\rho_{\tau,P}(Z)\right|
\ge
cB\min\left\{1,\frac1{\varepsilon n\tau}\right\}.
\]
The same lower bound holds for $(\varepsilon,\delta)$-DP estimators whenever $\delta\le c_2/k$, where $k=\lceil4np\rceil$ and $p=c_1\min\{\tau,1/(\varepsilon n)\}$.
\end{theorem}

\begin{proof}
Let $P_0$ put all mass at $0$. Let $P_1$ put mass $p$ at $B$ and mass $1-p$ at $0$, where
\[
p=c_1\min\left\{\tau,\frac1{\varepsilon n}\right\}
\]
for a small numerical constant $c_1>0$. Since $p\le\tau$, the Rockafellar-Uryasev formula gives
\[
\rho_{\tau,P_0}(Z)=0,
\qquad
\rho_{\tau,P_1}(Z)=\frac{pB}{\tau}=:\Delta.
\]
Indeed, for $X=B\1\{U=1\}$ with $U\sim\operatorname{Bernoulli}(p)$ and $p\le\tau$,
\[
\rho_\tau(X)
=
\inf_{\eta\in[0,B]}
\left\{\eta+\frac p\tau(B-\eta)\right\}
=
\inf_{\eta\in[0,B]}
\left\{\frac{pB}{\tau}+\eta\left(1-\frac p\tau\right)\right\}
=
\frac{pB}{\tau}.
\]

Suppose for contradiction that an $\varepsilon$-DP estimator satisfies expected error at most $\Delta/16$ under both $P_0^n$ and $P_1^n$. Define the event
\[
E=\{\widehat r(S)\le\Delta/2\}.
\]
By Markov's inequality,
\[
P_0^n(E)\ge 7/8,
\qquad
P_1^n(E)\le 1/8.
\]
Under $P_1^n$, let $K$ be the number of nonzero observations. Then $K\sim\operatorname{Binomial}(n,p)$ and $\E K=np$. Let $k=\lceil4np\rceil$. Markov's inequality gives $P_1^n(K\le k)\ge3/4$. On the event $\{K\le k\}$, the realized dataset differs from the all-zero dataset in at most $k$ entries. By group privacy,
\begin{align*}
\Pp(\widehat r(S)\le\Delta/2\mid S)
&\ge
 e^{-k\varepsilon}
\Pp(\widehat r(0^n)\le\Delta/2)\\
&\ge e^{-k\varepsilon}P_0^n(E).
\end{align*}
This is the pure-DP comparison. Since $k\varepsilon\le 4c_1+\varepsilon$ and $c_1$ is chosen small, the right-hand side is bounded below by a numerical constant larger than $1/4$. Consequently $P_1^n(E)>1/8$, contradicting the previous display.

For approximate DP, the same argument uses the group-privacy bound
\[
\Pp(\widehat r(S)\in E)
\ge
 e^{-k\varepsilon}\Pp(\widehat r(0^n)\in E)-k\delta.
\]
Taking $\delta\le c_2/k$ for a sufficiently small universal constant $c_2$ preserves the contradiction after reducing constants. Therefore some distribution among $P_0,P_1$ has expected error at least $\Delta/16$, and
\[
\Delta
=
\frac{pB}{\tau}
=
 c_1B\min\left\{1,\frac1{\varepsilon n\tau}\right\}.
\]
\end{proof}

\begin{theorem}[Scalar private plug-in upper bound]
\label{thm:scalar-private-plugin-upper}
Let $Z\in[0,B]$ and let $\widehat\rho_{\tau,S}=\rho_{\tau,\widehat P_n}(Z)$ be empirical CVaR. Define
\[
\widetilde\rho_\tau(S)
=
\Pi_{[0,B]}\left(\widehat\rho_{\tau,S}+\xi\right),
\qquad
\xi\sim \operatorname{Lap}\left(\frac{\Delta_\tau}{\varepsilon}\right),
\qquad
\Delta_\tau=B\min\left\{1,\frac1{n\tau}\right\}.
\]
Then $\widetilde\rho_\tau$ is $\varepsilon$-DP and, for $\varepsilon\in(0,1]$,
\[
\sup_{P:Z\in[0,B]}
\E\left|\widetilde\rho_\tau(S)-\rho_{\tau,P}(Z)\right|
\le
\sup_{P:Z\in[0,B]}
\E\left|\widehat\rho_{\tau,S}-\rho_{\tau,P}(Z)\right|
+
B\min\left\{1,\frac1{\varepsilon n\tau}\right\}.
\]
\end{theorem}

\begin{proof}
By Lemma 3.1, $\widehat\rho_{\tau,S}$ has one-record sensitivity at most $\Delta_\tau$. The Laplace mechanism therefore gives $\varepsilon$-DP. Since $\rho_{\tau,P}(Z)\in[0,B]$, projection onto $[0,B]$ cannot increase absolute error, so
\[
\left|
\Pi_{[0,B]}(\widehat\rho_{\tau,S}+\xi)-\rho_{\tau,P}(Z)
\right|
\le
\left|\widehat\rho_{\tau,S}-\rho_{\tau,P}(Z)\right|+|\xi|.
\]
Taking expectations gives the nonprivate empirical error plus $\E|\xi|=\Delta_\tau/\varepsilon$. The projected estimator also has error at most $B$. Since $\varepsilon\le1$,
\[
\min\left\{
B,\frac{\Delta_\tau}{\varepsilon}
\right\}
\le
B\min\left\{1,\frac1{\varepsilon n\tau}\right\},
\]
which proves the display.
\end{proof}

\begin{theorem}[Complete scalar private CVaR minimax rate]
\label{thm:complete-scalar-private-cvar}
Let
\[
\mathfrak R^{\varepsilon,0,\mathrm{scal}}_{n,\tau}
:=
\inf_{\widehat r\in\varepsilon\text{-}\DP}
\sup_{P:Z\in[0,B]}
\E_{S\sim P^n}\left|\widehat r(S)-\rho_{\tau,P}(Z)\right|.
\]
For $\varepsilon\in(0,1]$, under pure DP,
\[
\mathfrak R^{\varepsilon,0,\mathrm{scal}}_{n,\tau}
=
\Theta\left(
B\min\left\{
1,
\frac1{\sqrt{n\tau}}
+
\frac1{\varepsilon n\tau}
\right\}
\right).
\]
The same complete rate holds for approximate DP under the small-$\delta$ condition in Theorem~\ref{thm:scalar-privacy-lower}.
\end{theorem}

\begin{proof}
For the upper bound, Lemma~\ref{lem:scalar-cvar-concentration} implies
\[
\sup_{P:Z\in[0,B]}
\E\left|\widehat\rho_{\tau,S}-\rho_{\tau,P}(Z)\right|
\le
CB\min\left\{1,\frac1{\sqrt{n\tau}}\right\}.
\]
Combining this with the plug-in private estimator in Theorem~\ref{thm:scalar-private-plugin-upper} gives
\[
\mathfrak R^{\varepsilon,0,\mathrm{scal}}_{n,\tau}
\le
CB\min\left\{
1,
\frac1{\sqrt{n\tau}}
+
\frac1{\varepsilon n\tau}
\right\}.
\]

For the nonprivate lower bound, use Bernoulli losses. When $n\tau\lesssim1$, compare Bernoulli masses $p_0=\tau/2$ and $p_1=\tau$; the $n$-sample KL divergence is bounded by a universal constant and the CVaR gap is $B/2$. When $n\tau\gtrsim1$, compare Bernoulli masses $p_0=\tau/2$ and $p_1=\tau/2+c\sqrt{\tau/n}$, with $c>0$ small enough that $p_1\le\tau$. For $X=B\operatorname{Bernoulli}(p)$ with $p\le\tau$, $\rho_\tau(X)=pB/\tau$, so the CVaR gap is $cB/\sqrt{n\tau}$ in the second regime. In both regimes, Le Cam's method gives
\[
\mathfrak R^{\varepsilon,0,\mathrm{scal}}_{n,\tau}
\ge
cB\min\left\{1,\frac1{\sqrt{n\tau}}\right\}.
\]
Theorem~\ref{thm:scalar-privacy-lower} gives the privacy lower bound
\[
\mathfrak R^{\varepsilon,0,\mathrm{scal}}_{n,\tau}
\ge
cB\min\left\{1,\frac1{\varepsilon n\tau}\right\}.
\]
Taking the maximum of the two lower bounds and using $\max\{a,b\}\ge(a+b)/2$, with the trivial cap by $B$, proves the displayed lower bound. The approximate-DP statement follows from the small-$\delta$ extension in Theorem~\ref{thm:scalar-privacy-lower}.
\end{proof}

\paragraph{Approximate-DP regime.}
In Theorem 3.2, $p=c_1\min\{\tau,1/(\varepsilon n)\}$ and $k=\lceil4np\rceil$, so
\[
k=O\left(\min\left\{n\tau,\frac1\varepsilon\right\}\right).
\]
The approximate-DP lower bound requires $\delta\le c/k$. Thus standard choices such as $\delta=n^{-2}$ satisfy the condition in the $p=\tau$ regime once constants are fixed, since $k\lesssim n\tau\le n$. In the privacy-dominated regime, where $k\lesssim1/\varepsilon$, the same choice satisfies the condition for constant $\varepsilon$, or more generally whenever $n^{-2}\lesssim\varepsilon$.

\begin{remark}[Private quantiles are not enough]
In the proof above the minimizing threshold is known: $\eta=0$. Thus the lower bound is not caused by privately estimating a quantile. It is caused by privately estimating the tail average. This separates private CVaR learning from private quantile estimation and from general private estimation phenomena \citep{Smith2011}.
\end{remark}

\section{Finite-class private CVaR learning}

For finite classes, the decomposition can be closed completely. The ordinary CVaR statistical term has effective sample size $n\tau$, while the exponential mechanism contributes the privacy price at effective private sample size $\varepsilon n\tau$.
For a distribution $P$, write the ordinary nonprivate CVaR uniform-convergence term as
\[
\mathfrak G_{n,\tau}(\mathcal F,P)
:=
2\E_S\sup_{f\in\mathcal F}
\left|
\rho_{\tau,P}(f)-\widehat\rho_{\tau,S}(f)
\right|.
\]

\begin{lemma}[Scalar empirical CVaR concentration]
\label{lem:scalar-cvar-concentration}
Let $X_1,\ldots,X_n$ be iid copies of a random variable $X\in[0,B]$, and let $\widehat\rho_{\tau,n}$ be the empirical CVaR at tail mass $\tau$. There is a universal constant $C>0$ such that, for every $u\ge1$,
\[
\Pp\left(
\left|\widehat\rho_{\tau,n}-\rho_\tau(X)\right|
>
CB\min\left\{1,\sqrt{\frac{u}{n\tau}}+\frac{u}{n\tau}\right\}
\right)
\le
4e^{-u}.
\]
In particular, after increasing $C$,
\[
\Pp\left(
\left|\widehat\rho_{\tau,n}-\rho_\tau(X)\right|
>
CB\min\left\{1,\sqrt{\frac{u}{n\tau}}\right\}
\right)
\le
4e^{-u}
\]
for all $u\ge1$.
\end{lemma}

\begin{proof}
Write
\[
\phi(\eta)=\eta+\frac1\tau P(X-\eta)_+,
\qquad
\widehat\phi(\eta)=\eta+\frac1\tau P_n(X-\eta)_+,
\]
where $P_n$ denotes the empirical measure, and let $h_\eta(x)=(x-\eta)_+$.
The following relative Bernstein inequality is used for the one-dimensional excess-threshold class $\mathcal H=\{h_\eta:\eta\in[0,B]\}$: with probability at least $1-2e^{-u}$, simultaneously for all $\eta\in[0,B]$,
\[
(P_n-P)h_\eta
\le
C_0\left(
\sqrt{\frac{B\,P h_\eta\,u}{n}}
+
\frac{Bu}{n}
\right),
\tag{4.1}
\]
and
\[
(P-P_n)h_\eta
\le
C_0\left(
\sqrt{\frac{B\,P_n h_\eta\,u}{n}}
+
\frac{Bu}{n}
\right).
\tag{4.2}
\]
For completeness, Appendix~\ref{app:relative-bernstein-thresholds} proves this relative Bernstein specialization; scalar empirical CVaR concentration inequalities of this type also go back to \citet{Brown2007} and were sharpened by \citet{ThomasLearnedMiller2019}.

Let $\eta^\star\in\argmin_{\eta\in[0,B]}\phi(\eta)$. Since $\rho_\tau(X)\le B$ and $\eta^\star\ge0$,
\[
P h_{\eta^\star}
=
\tau(\rho_\tau(X)-\eta^\star)
\le
\tau B.
\]
On the event (4.1),
\[
\widehat\rho_{\tau,n}-\rho_\tau(X)
\le
\widehat\phi(\eta^\star)-\phi(\eta^\star)
=
\frac1\tau(P_n-P)h_{\eta^\star}
\le
C_0B\left(
\sqrt{\frac{u}{n\tau}}
+
\frac{u}{n\tau}
\right).
\]

For the reverse direction, let $\widehat\eta\in\argmin_{\eta\in[0,B]}\widehat\phi(\eta)$. Since $\widehat\phi(\widehat\eta)\le\widehat\phi(B)=B$ and $\widehat\eta\ge0$,
\[
P_nh_{\widehat\eta}
=
\tau(\widehat\rho_{\tau,n}-\widehat\eta)
\le
\tau B.
\]
On the event (4.2),
\[
\rho_\tau(X)-\widehat\rho_{\tau,n}
\le
\phi(\widehat\eta)-\widehat\phi(\widehat\eta)
=
\frac1\tau(P-P_n)h_{\widehat\eta}
\le
C_0B\left(
\sqrt{\frac{u}{n\tau}}
+
\frac{u}{n\tau}
\right).
\]
Combining the two one-sided bounds gives the first display. The second follows because both CVaR values lie in $[0,B]$: if $u\le n\tau$, the linear term is dominated by the square-root term, while if $u>n\tau$, the trivial bound $|\widehat\rho_{\tau,n}-\rho_\tau(X)|\le B$ applies.
\end{proof}

\begin{theorem}[Finite-class nonprivate CVaR uniform convergence]
\label{thm:finite-nonprivate-cvar}
There is a universal constant $C>0$ such that for every finite class $\mathcal F$ with $|\mathcal F|=M$ and losses in $[0,B]$,
\[
\sup_P
\mathfrak G_{n,\tau}(\mathcal F,P)
\le
CB\min\left\{1,\sqrt{\frac{\log(2M)}{n\tau}}\right\}.
\]
Moreover, this rate is minimax optimal up to constants for ordinary nonprivate CVaR learning over finite classes.
\end{theorem}

\begin{proof}
Apply Lemma~\ref{lem:scalar-cvar-concentration} to the scalar loss $\ell(f;Z)$ for each fixed $f\in\mathcal F$ and take a union bound. For $s\ge0$, with probability at least $1-4e^{-s}$,
\[
\sup_{f\in\mathcal F}
\left|
\rho_{\tau,P}(f)-\widehat\rho_{\tau,S}(f)
\right|
\le
CB\min\left\{
1,
\sqrt{\frac{\log(4M)+s}{n\tau}}
\right\}.
\]
Integrating this tail bound gives
\[
\E_S\sup_{f\in\mathcal F}
\left|
\rho_{\tau,P}(f)-\widehat\rho_{\tau,S}(f)
\right|
\le
CB\min\left\{
1,
\sqrt{\frac{\log(2M)}{n\tau}}
\right\},
\]
after adjusting the universal constant. Multiplying by the factor $2$ in the definition of $\mathfrak G_{n,\tau}$ gives the upper bound.

For the lower bound, use the canonical tail embedding with no privacy constraint. Ordinary finite-class learning with $m$ iid informative observations has minimax excess risk at least
\[
cB\min\left\{1,\sqrt{\frac{\log M}{m}}\right\}
\]
by the standard Fano--Assouad finite-class lower bound. Let $m=\lceil4n\tau\rceil$ and suppose, toward contradiction, that an ordinary nonprivate CVaR learner on $n$ samples achieved excess risk $\alpha_n$ on every embedded finite-class instance. Given $m$ ordinary iid samples, draw independent activation indicators $T_1,\ldots,T_n\sim\operatorname{Bernoulli}(\tau)$ and set $K=\sum_iT_i$. If $K\le m$, place the first $K$ ordinary samples into the active slots and fill inactive slots with the dummy zero-loss point; if $K>m$, fill the remaining active slots with the dummy point as well. No privacy parameters are involved here; this is only a coupling argument. The constructed sample has total variation distance at most $\Pp(K>m)\le\exp(-c n\tau)$ from the ideal embedded CVaR sample, and under the embedding the CVaR excess equals ordinary excess exactly. Thus the CVaR learner would give an ordinary finite-class learner with $m\asymp n\tau$ samples and excess at most $\alpha_n+B\exp(-c n\tau)$. Comparing with the ordinary finite-class lower bound and absorbing the exponentially small residual gives
\[
cB\min\left\{1,\sqrt{\frac{\log M}{n\tau}}\right\},
\]
with constants adjusted for the binomial coupling. This is the nonprivate version of the tail-sample transfer theorem proved formally in Theorem~\ref{thm:tail-sample-transfer}.
\end{proof}

\begin{theorem}[Finite-class private upper bound]
\label{thm:finite-private-upper}
There is a universal constant $C>0$ such that the following holds. Let $\mathcal F$ be finite with $|\mathcal F|=M$, and suppose $\ell(f;z)\in[0,B]$. The exponential mechanism of \citet{McSherryTalwar2007}, with score
\[
q(S,f)=-\widehat\rho_{\tau,S}(f)
\]
and sensitivity $B/(n\tau)$ is $\varepsilon$-DP and returns $\widehat f$ satisfying
\[
\E\left[\rho_{\tau,P}(\widehat f)-\inf_{f\in\mathcal F}\rho_{\tau,P}(f)\right]
\le
\mathfrak G_{n,\tau}(\mathcal F,P)
+
CB\min\left\{1,\frac{\log(2M)}{\varepsilon n\tau}\right\}.
\]
Consequently,
\[
\sup_P
\E\left[\rho_{\tau,P}(\widehat f)-\inf_{f\in\mathcal F}\rho_{\tau,P}(f)\right]
\le
CB\min\left\{1,
\sqrt{\frac{\log(2M)}{n\tau}}
+
\frac{\log(2M)}{\varepsilon n\tau}
\right\}.
\]
\end{theorem}

\begin{proof}
The sensitivity lemma implies that $q(S,f)$ has sensitivity at most $B/(n\tau)$. Therefore the exponential mechanism sampling
\[
\Pp(\widehat f=f)
\propto
\exp\left(-\frac{\varepsilon n\tau}{2B}\widehat\rho_{\tau,S}(f)\right)
\]
is $\varepsilon$-DP. The standard expected utility guarantee for the exponential mechanism \citep{McSherryTalwar2007} yields
\[
\E\left[\widehat\rho_{\tau,S}(\widehat f)-\min_{f\in\mathcal F}\widehat\rho_{\tau,S}(f)\mid S\right]
\le
\frac{2B(\log M+1)}{\varepsilon n\tau}.
\]
Let $f^\star\in\argmin_f\rho_{\tau,P}(f)$. Then
\begin{align*}
\rho_{\tau,P}(\widehat f)-\rho_{\tau,P}(f^\star)
&\le
\left[\rho_{\tau,P}(\widehat f)-\widehat\rho_{\tau,S}(\widehat f)\right]
+
\left[\widehat\rho_{\tau,S}(\widehat f)-\widehat\rho_{\tau,S}(f^\star)\right]\\
&\quad+
\left[\widehat\rho_{\tau,S}(f^\star)-\rho_{\tau,P}(f^\star)\right]\\
&\le
2\sup_f|\rho_{\tau,P}(f)-\widehat\rho_{\tau,S}(f)|
+
\widehat\rho_{\tau,S}(\widehat f)-\min_f\widehat\rho_{\tau,S}(f).
\end{align*}
Taking expectations gives the same display with privacy term
\[
V_{\mathrm{priv}}:=\frac{2B(\log M+1)}{\varepsilon n\tau}.
\]
The expected excess risk is also at most $B$, because all losses lie in $[0,B]$. Hence it is bounded by both $\mathfrak G_{n,\tau}(\mathcal F,P)+V_{\mathrm{priv}}$ and $B$, and therefore by $\mathfrak G_{n,\tau}(\mathcal F,P)+\min\{B,V_{\mathrm{priv}}\}$.
Writing $a=(\log M+1)/(\varepsilon n\tau)$, the displayed bound follows from
\[
\min\{B,2Ba\}\le 2B\min\{1,a\}.
\]
\end{proof}

The preceding theorem says that finite-class private CVaR learning decomposes into the ordinary term $B\sqrt{\log(2M)/(n\tau)}$ and the privacy price $B\log(2M)/(\varepsilon n\tau)$, up to universal constants. The next theorem shows that the privacy dependence is minimax optimal up to constants and the trivial range cap.

\begin{theorem}[Finite-class minimax lower bound for the privacy price]
\label{thm:finite-privacy-lower}
There are universal constants $c,c',c_0>0$ such that for every $M\ge2$, $n\ge1$, $\tau\in(0,1]$, $B>0$, and $\varepsilon\in(0,1]$, there is a class $\mathcal F$ with $|\mathcal F|=M$ and losses in $[0,B]$ such that every $\varepsilon$-DP learner $\mathcal A$ returning an element of $\mathcal F$ satisfies
\[
\sup_P
\E\left[
\rho_{\tau,P}(\mathcal A(S))-
\inf_{f\in\mathcal F}\rho_{\tau,P}(f)
\right]
\ge
cB\min\left\{1,\frac{\log M}{\varepsilon n\tau}\right\}.
\]
The same extension holds for $(\varepsilon,\delta)$-DP under the small-$\delta$ condition $\delta\le c'/(k\sqrt M)$, where $k=\lceil4np\rceil$ and $p=c_0\min\{\tau,\log M/(\varepsilon n)\}$.
\end{theorem}

\begin{proof}
First suppose $M$ is larger than a sufficiently large universal constant $M_0$. Let $\mathcal Z=\{0,1,\ldots,M\}$ and $\mathcal F=\{f_1,\ldots,f_M\}$. Define
\[
\ell(f_r;z)=B\1\{z\in[M],\ z\ne r\}.
\]
The point $z=0$ is a null observation with zero loss for every predictor. For each $j\in[M]$, let $P_j$ be the distribution with
\[
P_j(Z=j)=p,
\qquad
P_j(Z=0)=1-p,
\]
where
\[
p=c_0\min\left\{\tau,\frac{\log M}{\varepsilon n}\right\}
\]
for a sufficiently small universal constant $c_0>0$. Since $p\le\tau$, under $P_j$ the predictor $f_j$ has zero loss always, while every $f_r$ with $r\ne j$ has loss $B$ with probability $p$ and loss $0$ otherwise. Therefore
\[
\rho_{\tau,P_j}(f_j)=0,
\qquad
\rho_{\tau,P_j}(f_r)=\frac{pB}{\tau}
\quad(r\ne j).
\]
Indeed, for $X=B\1\{U=1\}$ with $U\sim\operatorname{Bernoulli}(p)$ and $p\le\tau$,
\[
\rho_\tau(X)
=
\inf_{\eta\in[0,B]}\left\{\eta+\frac p\tau(B-\eta)\right\}
=
\frac{pB}{\tau}.
\]
Thus the excess-risk gap for every wrong predictor under $P_j$ is
\[
\Delta=\frac{Bp}{\tau}.
\]

Suppose, toward contradiction, that the expected excess risk is at most $\Delta/4$ under every $P_j$. Then
\[
\Pp_{S\sim P_j^n}(\mathcal A(S)=f_j)\ge\frac34
\qquad\text{for every }j.
\]
Let $S_0=(0,\ldots,0)$ and write $q_j=\Pp(\mathcal A(S_0)=f_j)$. Then $\sum_jq_j=1$. Under $P_j^n$, let $K_j=|\{i:Z_i=j\}|$. Then $K_j\sim\operatorname{Binomial}(n,p)$, and with $k=\lceil4np\rceil$, Markov's inequality gives $\Pp(K_j>k)\le1/4$. On $\{K_j\le k\}$, the sample differs from $S_0$ in at most $k$ entries, so pure group privacy gives
\[
\Pp(\mathcal A(S)=f_j)\le e^{k\varepsilon}q_j.
\]
Consequently,
\[
\Pp_{S\sim P_j^n}(\mathcal A(S)=f_j)
\le
\frac14+e^{k\varepsilon}q_j.
\]
Averaging over $j$ yields
\[
\frac1M\sum_{j=1}^M
\Pp_{S\sim P_j^n}(\mathcal A(S)=f_j)
\le
\frac14+\frac{e^{k\varepsilon}}{M}.
\]
Because $k\varepsilon\le4c_0\log M+1$ and $M\ge M_0$, choosing $c_0$ small and $M_0$ large ensures $e^{k\varepsilon}/M\le1/4$. The average success probability is therefore at most $1/2$, contradicting the lower bound $3/4$ for every term.

For $(\varepsilon,\delta)$-DP, group privacy gives
\[
\Pp(\mathcal A(S)=f_j)
\le
e^{k\varepsilon}q_j
+
k e^{k\varepsilon}\delta.
\]
Averaging adds the term $k e^{k\varepsilon}\delta$. Since the previous choice of constants gives $e^{k\varepsilon}\le\sqrt M$, it suffices to require $\delta\le c'/(k\sqrt M)$ for a small universal constant $c'$.

For $2\le M<M_0$, the same conclusion follows, after changing constants, from the two-hypothesis construction on $\{0,1\}$, augmented with $M-2$ dummy predictors $g$ whose loss is identically $B$. Since every dummy predictor has excess at least the hard gap under both hard distributions, any learner with expected excess below a fixed fraction of the gap must solve the original two-hypothesis problem with constant success probability. Since $\log M\asymp1$ in this range, the displayed rate reduces to the two-point rate. Finally,
\[
\Delta
=
\frac{pB}{\tau}
=
c_0B\min\left\{1,\frac{\log M}{\varepsilon n\tau}\right\},
\]
which proves the theorem.
\end{proof}

\paragraph{Approximate-DP regime.}
In the finite-class packing, $p=c_0\min\{\tau,\log M/(\varepsilon n)\}$ and $k=\lceil4np\rceil$, so
\[
k=O\left(\min\left\{n\tau,\frac{\log M}{\varepsilon}\right\}\right).
\]
The approximate-DP lower bound assumes $\delta\le c/(k\sqrt M)$. This is the regime in which approximate-DP group privacy still distinguishes the $M$-way packing. For polynomial-size classes and the common choice $\delta=n^{-\omega(1)}$, this condition is satisfied after constants are fixed. For exponentially large $M$, the approximate-DP lower bound is intentionally weaker; the pure-DP lower bound remains unconditional.

\begin{theorem}[Complete finite-class private CVaR minimax rate]
\label{thm:complete-finite-private-cvar}
Let $\mathfrak R^{\varepsilon,0}_{n,\tau}(M)$ denote the distribution-free minimax excess CVaR risk over all finite classes of size $M$ with losses in $[0,B]$. Under pure $\varepsilon$-DP,
\[
\mathfrak R^{\varepsilon,0}_{n,\tau}(M)
=
\Theta\left(
B\min\left\{
1,
\sqrt{\frac{\log(2M)}{n\tau}}
+
\frac{\log(2M)}{\varepsilon n\tau}
\right\}
\right),
\]
with universal implicit constants. The same complete rate holds for approximate DP under the small-$\delta$ condition in Theorem~\ref{thm:finite-privacy-lower}.
\end{theorem}

\begin{proof}
The upper bound is Theorem~\ref{thm:finite-private-upper} combined with Theorem~\ref{thm:finite-nonprivate-cvar}.

For the lower bound, Theorem~\ref{thm:finite-nonprivate-cvar} gives an ordinary statistical lower bound
\[
a
:=
cB\min\left\{1,\sqrt{\frac{\log(2M)}{n\tau}}\right\},
\]
while Theorem~\ref{thm:finite-privacy-lower} gives a privacy lower bound
\[
b
:=
cB\min\left\{1,\frac{\log M}{\varepsilon n\tau}\right\}
\asymp
cB\min\left\{1,\frac{\log(2M)}{\varepsilon n\tau}\right\},
\]
where the last comparison uses $M\ge2$.
Since both lower bounds apply to the same distribution-free minimax problem over size-$M$ classes,
\[
\mathfrak R^{\varepsilon,0}_{n,\tau}(M)\ge \max\{a,b\}.
\]
For nonnegative $x,y$, $\max\{x,y\}\ge (x+y)/2$, and after applying the trivial cap $B$ this yields
\[
\max\{a,b\}
\ge
c'B\min\left\{
1,
\sqrt{\frac{\log(2M)}{n\tau}}
+
\frac{\log(2M)}{\varepsilon n\tau}
\right\}
\]
for another universal constant $c'>0$. This proves the lower bound.
\end{proof}

Thus finite classes give the anchor theorem: the ordinary tail sample size $n\tau$ and the effective private tail sample size $\varepsilon n\tau$ appear in the same decomposition. In particular, the privacy component of this complete rate is
\[
B\min\left\{1,\frac{\log(2|\mathcal F|)}{\varepsilon n\tau}\right\},
\]
up to universal constants.

\section{Convex private CVaR learning as modular transfer}

Now consider a convex parameter set $W\subset\R^d$. The purpose of this section is modular: isolate the CVaR-induced privacy degradation from the ambient private-SCO geometry. The upper bound reduces private CVaR optimization to private SCO on a lifted objective; the lower bound in Section~6 transfers private-SCO hardness back through the tail embedding.

\begin{assumption}[Bounded convex Lipschitz losses]
The set $W$ is closed, convex, and has Euclidean diameter at most $D$. The loss $\ell(w;z)$ is convex in $w$, $G$-Lipschitz in $w$, and takes values in $[0,B]$.
\end{assumption}

For a scale parameter $\lambda>0$, set $u=\eta/\lambda$, define
\[
\Theta_\lambda=W\times[0,B/\lambda]\subset\R^{d+1},
\]
and write the scaled lifted loss as
\[
g_\lambda(w,u;z)
=
\lambda u+\frac1\tau(\ell(w;z)-\lambda u)_+.
\]
Then minimizing $\E g_\lambda(w,u;Z)$ over $(w,u)\in\Theta_\lambda$ is exactly the Rockafellar--Uryasev CVaR problem with the threshold variable $\eta=\lambda u$.

\begin{lemma}[Scaled lifted convexity and Lipschitzness]
Under Assumption 5.1, $g_\lambda(\cdot;z)$ is convex on $\Theta_\lambda$ and is $L_{\tau,\lambda}$-Lipschitz with
\[
L_{\tau,\lambda}\le\frac{\sqrt{G^2+\lambda^2}}{\tau}.
\]
The diameter of $\Theta_\lambda$ is at most
\[
D_{\Theta,\lambda}\le\sqrt{D^2+\frac{B^2}{\lambda^2}}.
\]
\end{lemma}

\begin{proof}
The map $(v,u)\mapsto(v-\lambda u)_+$ is convex and nondecreasing in $v$, while $w\mapsto\ell(w;z)$ is convex. Hence $g_\lambda$ is convex. For a subgradient, let $s\in[0,1]$ be a subgradient of $r\mapsto r_+$ at $r=\ell(w;z)-\lambda u$. Then
\[
\partial_w g_\lambda=s\tau^{-1}\partial_w\ell(w;z),
\qquad
\partial_u g_\lambda=\lambda(1-s\tau^{-1}).
\]
Therefore
\[
\|\partial_w g_\lambda\|_2\le G/\tau,
\qquad
|\partial_u g_\lambda|\le \lambda/\tau,
\]
where the latter uses $\tau\le1$. The Lipschitz bound follows. The diameter bound is immediate from the product geometry.
\end{proof}

\subsection{Private SCO as an imported interface}

The convex results are intentionally modular. Rather than reproving Euclidean private-SCO theory, the relevant result is stated as a precise interface, and the Rockafellar--Uryasev CVaR lift is then verified to satisfy that interface. This separation isolates the CVaR-specific phenomenon: the replacement of the full sample size $n$ by the effective tail sample size $n\tau$, and of the private sample size $\varepsilon n$ by the effective private tail sample size $\varepsilon n\tau$.

All logarithmic factors hidden in $\widetilde O(\cdot)$ and $\widetilde\Omega(\cdot)$ in this subsection are exactly those inherited from the cited Euclidean approximate-DP SCO interface. The CVaR lifting and tail-embedding reductions introduce no additional polynomial dependence on $n,\tau,\varepsilon,\delta$, or $d$.

\begin{theorem}[Euclidean approximate-DP SCO interfaces]
\label{thm:sco-interface}
Let $\Theta\subset\R^{d_\Theta}$ be closed and convex with Euclidean diameter at most $D_\Theta$. For a distribution $Q$ over losses $g(\cdot;Z)$, write
\[
    \mathcal L_Q(\theta)=\E_{Z\sim Q}g(\theta;Z),
    \qquad
    \mathcal L_Q^\star=\inf_{\theta\in\Theta}\mathcal L_Q(\theta).
\]
The cited Euclidean SCO results are usually stated in terms of a radius $M_\Theta=\sup_{\theta\in\Theta}\|\theta\|_2$. Since excess risk and differential privacy are invariant under translations of the parameter domain, $\Theta$ can be translated before the interface is applied, so that $M_\Theta\le D_\Theta$. Thus the radius-based statements imply the diameter-based form used below, up to universal constants.
The excess-risk criterion is expected population excess,
\[
    \E_{S,A}\bigl[\mathcal L_Q(A(S))-\mathcal L_Q^\star\bigr],
    \qquad S\sim Q^n,
\]
where the expectation is over both the sample and the learner's internal randomness. Privacy is record-level central privacy with respect to the $n$ sampled losses.

\textbf{Upper interface.}
Suppose that, for every $z$, $g(\cdot;z)$ is convex and $L_\Theta$-Lipschitz on $\Theta$. For every $\varepsilon\in(0,1]$ and $0<\delta\le n^{-2}$, the non-smooth Euclidean private-SCO theorem of \citet[Theorem~4.4]{BassilyFeldmanTalwarThakurta2019} gives a central $(\varepsilon,\delta)$-DP algorithm $A_{\rm SCO}$ satisfying
\[
    \E_{S,A}\bigl[
        \mathcal L_Q(A_{\rm SCO}(S))-\mathcal L_Q^\star
    \bigr]
    \le
    \widetilde O\!\left(
        L_\Theta D_\Theta
        \left[
            \frac1{\sqrt n}
            +
            \frac{\sqrt{d_\Theta\log(1/\delta)}}{\varepsilon n}
        \right]
    \right).
\]
The non-smooth implementation uses Moreau--Yosida smoothing, equivalently proximal access to the individual losses. Because the theorem is used only as a statistical rate primitive, no running-time claim is made.

\textbf{Lower interface.}
For the privacy-dependent term, the standard reduction from private ERM lower bounds to private SCO lower bounds is used, as recorded by \citet[Appendix~C]{BassilyFeldmanTalwarThakurta2019}, together with the Euclidean private-ERM lower bounds of \citet{BassilySmithThakurta2014}. The nonprivate $1/\sqrt m$ term is the standard stochastic convex optimization lower bound. These ingredients are combined as a $\widetilde\Omega$ lower-bound interface, with the tilde hiding only the logarithmic losses from this imported SCO lower-bound reduction.

For every $m\ge1$, $d_\Theta\ge1$, $D_\Theta>0$, $L_\Theta>0$, $\varepsilon\in(0,1]$, and $0<\delta\le m^{-2}$, this interface gives a distribution over shifted linear losses on a Euclidean domain of diameter $D_\Theta$ such that the losses are convex, $L_\Theta$-Lipschitz, and have range contained in an interval of length at most $L_\Theta D_\Theta$, while every central $(\varepsilon,\delta)$-DP learner satisfies
\[
    \sup_Q
    \E_{S,A}\bigl[
        \mathcal L_Q(A(S))-\mathcal L_Q^\star
    \bigr]
    \ge
    \widetilde\Omega\!\left(
        L_\Theta D_\Theta
        \min\left\{
            1,\,
            \frac1{\sqrt m}
            +
            \frac{\sqrt{d_\Theta}}{\varepsilon m}
        \right\}
    \right).
\]
The lower-bound class is convex and bounded without clipping: it is obtained by shifting linear losses by a parameter-independent constant. Such shifts preserve excess risk.
\end{theorem}

\paragraph{Verification for the lifted CVaR problem.}
For the lifted CVaR loss
\[
    g_\lambda(w,u;z)
    =
    \lambda u+\frac1\tau(\ell(w;z)-\lambda u)_+,
    \qquad
    (w,u)\in\Theta_\lambda:=W\times[0,B/\lambda],
\]
Lemma~5.2 gives
\[
    d_\Theta=d+1,\qquad
    L_\Theta=L_{\tau,\lambda}
    \le \frac{\sqrt{G^2+\lambda^2}}{\tau},
    \qquad
    D_\Theta=D_{\Theta,\lambda}
    \le
    \sqrt{D^2+\frac{B^2}{\lambda^2}}.
\]
Moreover $0\le g_\lambda(w,u;z)\le B/\tau$, because $\ell(w;z)\in[0,B]$ and $\lambda u\in[0,B]$. Thus the lifted CVaR problem satisfies the Euclidean SCO upper interface with dimension $d+1$. Applying the upper-interface algorithm to the lifted losses preserves record-level privacy, because each original record $z_i$ produces exactly one lifted loss $g_\lambda(\cdot;z_i)$.

\begin{theorem}[Private convex CVaR upper bound]
Under Assumption 5.1, for $\varepsilon\in(0,1]$ and $0<\delta\le n^{-2}$, there exists a central $(\varepsilon,\delta)$-DP learner returning $\widehat w\in W$ such that
\[
\E\left[
\rho_{\tau,P}(\widehat w)-\inf_{w\in W}\rho_{\tau,P}(w)
\right]
\le
\widetilde O\left(
\frac{GD+B}{\tau}
\left[
\frac1{\sqrt n}
+
\frac{\sqrt{(d+1)\log(1/\delta)}}{\varepsilon n}
\right]
\right).
\]
In particular, the private part is
\[
\widetilde O\left(
\frac{(GD+B)\sqrt{d\log(1/\delta)}}{\varepsilon n\tau}
\right)
\].
Under the natural normalization $B\asymp GD$, this scale is $\widetilde O(B\sqrt{d\log(1/\delta)}/(\varepsilon n\tau))$.
\end{theorem}

\begin{proof}
Apply the upper interface in Theorem~\ref{thm:sco-interface} to the lifted loss $g_\lambda$ over $\Theta_\lambda$. Let $\widehat\theta=(\widehat w,\widehat u)$ be the output. Since
\[
\rho_{\tau,P}(\widehat w)
\le
\E g_\lambda(\widehat w,\widehat u;Z)
\]
and
\[
\inf_{w\in W}\rho_{\tau,P}(w)
=
\inf_{(w,u)\in\Theta_\lambda}\E g_\lambda(w,u;Z),
\]
the CVaR excess risk is bounded by the lifted SCO population excess risk.

By Lemma~5.2,
\[
L_{\tau,\lambda}D_{\Theta,\lambda}
\le
\frac1\tau
\sqrt{(G^2+\lambda^2)\left(D^2+\frac{B^2}{\lambda^2}\right)}.
\]
The expression inside the square root equals
\[
G^2D^2+B^2+\frac{G^2B^2}{\lambda^2}+\lambda^2D^2.
\]
Choosing $\lambda^2=GB/D$ when $G,B,D>0$, and using a limiting argument in degenerate cases, gives
\[
L_{\tau,\lambda}D_{\Theta,\lambda}
\le
\frac{GD+B}{\tau}.
\]
Substituting this into Theorem~\ref{thm:sco-interface}, with $d_\Theta=d+1$, yields
\[
    \E\!\left[
        \rho_{\tau,P}(\widehat w)-\inf_{w\in W}\rho_{\tau,P}(w)
    \right]
    \le
    \widetilde O\!\left(
        \frac{GD+B}{\tau}
        \left[
            \frac1{\sqrt n}
            +
            \frac{\sqrt{(d+1)\log(1/\delta)}}{\varepsilon n}
        \right]
    \right).
\]
The private part is therefore
\[
    \widetilde O\!\left(
        \frac{(GD+B)\sqrt{d\log(1/\delta)}}{\varepsilon n\tau}
    \right).
\]
\end{proof}

\begin{remark}[Modular convex transfer]
The convex upper bound is deliberately modular. The rescaled Rockafellar--Uryasev lifting supplies the CVaR-specific reduction, while the Euclidean approximate-DP SCO interface supplies the ambient high-dimensional optimization guarantee. This factorization is useful: any improvement in the private-SCO interface can be substituted into the same CVaR lift, and the lift itself introduces no additional polynomial dependence on $n,\tau,\varepsilon,\delta$, or $d$. The CVaR-specific privacy dependence is the factor $1/(\varepsilon n\tau)$, and the lower bounds below show that this factor is unavoidable.
\end{remark}

\section{Tail embedding and minimax lower bounds}

The next result is the conceptual core. It identifies the canonical hard subproblem inside CVaR learning: ordinary expected-risk learning with only the tail fraction of samples carrying information.

\subsection{Exact embedding of expectation into CVaR}

Let $Q$ be a distribution on $\mathcal X$, let $a(w;x)\in[0,B]$ be an ordinary loss, and add a dummy point $x_\bot\notin\mathcal X$ with $a(w;x_\bot)=0$ for all $w$. Draw $T\sim\operatorname{Bernoulli}(\tau)$. If $T=1$, draw $Y=X\sim Q$; if $T=0$, set $Y=x_\bot$. Define $Z=(T,Y)$ and the CVaR loss
\[
\ell(w;Z)=T a(w;Y).
\]
Let $P_\tau(Q)$ denote the induced distribution on $Z$.

\begin{theorem}[Canonical tail embedding]
For every $w$,
\[
\rho_{\tau,P_\tau(Q)}(\ell(w;Z))
=
\E_{X\sim Q}[a(w;X)].
\]
Consequently, minimizing CVaR over the embedded distribution is exactly equivalent to minimizing ordinary expected risk under $Q$.
\end{theorem}

\begin{proof}
For $\eta=0$,
\[
\eta+\frac1\tau\E[(T a(w;Y)-\eta)_+]
=
\frac1\tau\E[T a(w;Y)]
=
\E_Q a(w;X).
\]
Now take $\eta\ge0$. Then
\begin{align*}
\eta+\frac1\tau\E[(T a(w;Y)-\eta)_+]
&=
\eta+\E_Q[(a(w;X)-\eta)_+]\\
&=\E_Q[\max\{a(w;X),\eta\}]
\ge
\E_Q a(w;X).
\end{align*}
For $\eta<0$,
\begin{align*}
\eta+\frac1\tau\E[(T a(w;Y)-\eta)_+]
&=
\eta+\E_Q[a(w;X)-\eta]+\frac{1-\tau}{\tau}(-\eta)\\
&=\E_Q a(w;X)+\frac{1-\tau}{\tau}(-\eta)
\ge
\E_Q a(w;X).
\end{align*}
Thus the infimum over $\eta$ equals $\E_Qa(w;X)$.
\end{proof}

\begin{remark}
The nonnegativity of $a$ is used in the embedding. If an imported ordinary-risk lower bound is stated for losses in $[-B,B]$, one first shifts every loss by the same $w$-independent constant to obtain losses in $[0,2B]$. This preserves excess risk and changes only constants.
\end{remark}

Figure~\ref{fig:tail-sample-embedding} illustrates the reduction: ordinary
informative examples become the active records in an embedded CVaR sample, while
inactive records carry zero loss.

\begin{figure}[t!]
\centering
\begin{adjustbox}{max width=\textwidth}
\begin{tikzpicture}[
  x=1cm,y=1cm,
  font=\sffamily\small,
  >=Latex,
  dot/.style={circle,inner sep=0pt,minimum size=2.3pt},
  bigdot/.style={circle,inner sep=0pt,minimum size=3.0pt},
  arrow/.style={->,line width=0.55pt,draw=black!45},
  box/.style={rounded corners=7pt,draw=black!11,fill=white},
  title/.style={font=\sffamily\bfseries\scriptsize,align=center},
  smalllabel/.style={font=\sffamily\scriptsize,align=center}
]
\definecolor{tailorange}{RGB}{203,92,47}
\definecolor{learnblue}{RGB}{86,135,184}
\definecolor{softblue}{RGB}{237,244,250}
\definecolor{bridgeblue}{RGB}{246,249,252}
\definecolor{softorange}{RGB}{253,240,234}
\definecolor{softgray}{RGB}{245,246,248}
\definecolor{signalgray}{RGB}{186,190,196}
\definecolor{navytext}{RGB}{32,45,62}

\begin{scope}[shift={(0,2.52)}]
  \draw[box,fill=softblue] (0,0) rectangle (4.20,1.75);
  \node[title] at (2.10,1.48) {Ordinary SCO instance};
  \foreach \x/\y in {.55/.92,.88/.56,1.18/1.16,1.52/.78,1.88/1.05,2.25/.62,2.58/1.22}
    \node[bigdot,fill=learnblue!85] at (\x,\y) {};
  \node[smalllabel,text width=1.70cm] at (3.18,.86) {$X_i\sim Q$\\[-1pt]$m\asymp n\tau$};
\end{scope}

\begin{scope}[shift={(4.85,2.00)}]
  \draw[box,fill=bridgeblue] (0,0) rectangle (5.45,2.28);
  \node[title] at (2.72,2.00) {Tail embedding $\mathcal E_\tau$};
  \node[smalllabel] at (2.72,1.63) {$n$ CVaR slots,\quad $T_i\sim\mathrm{Bernoulli}(\tau)$};
  \foreach \x/\c in {.45/signalgray,.88/tailorange,1.31/signalgray,1.74/signalgray,2.17/tailorange,2.60/signalgray,3.03/signalgray,3.46/tailorange,3.89/signalgray,4.32/signalgray,4.75/tailorange}
    \draw[rounded corners=2pt,fill=\c!22,draw=\c!75,line width=.45pt] (\x,.92) rectangle ++(.24,.42);
  \foreach \x in {1.00,2.29,3.58,4.87}
    \node[bigdot,fill=tailorange] at (\x,1.13) {};
  \foreach \x in {.57,1.43,1.86,2.72,3.15,4.01,4.44}
    \node[dot,fill=signalgray!70] at (\x,1.13) {};
  \node[font=\sffamily\tiny,align=center,text=tailorange!90!black] at (1.45,.58) {active $(T_i=1)$};
  \node[font=\sffamily\tiny,align=center,text=black!65] at (3.35,.30) {inactive $(T_i=0):\ Y_i=x_\bot,\ \ell=0$};
\end{scope}

\draw[arrow] (4.34,3.38) -- (4.76,3.20);
\draw[arrow] (7.58,2.00) -- (7.58,1.58);

\begin{scope}[shift={(0,0)}]
  \draw[box,fill=softgray] (0,0) rectangle (10.30,1.55);
  \node[title] at (1.72,1.28) {CVaR instance};
  \draw[rounded corners=5pt,fill=white,draw=black!8] (.35,.22) rectangle (3.95,1.08);
  \foreach \x/\y in {.62/.37,1.00/.30,1.38/.42,1.78/.34,2.18/.40,2.58/.31,2.98/.43,3.38/.36,3.72/.44}
    \node[dot,fill=signalgray!70] at (\x,\y) {};
  \foreach \x/\y in {.78/.82,1.42/.94,2.05/.78,2.70/.98,3.35/.84}
    \node[bigdot,fill=tailorange] at (\x,\y) {};
  \node[smalllabel,anchor=west] at (4.35,.98) {$Z_i=(T_i,Y_i)$,\quad $\ell(w;Z_i)=T_i a(w;Y_i)$};
  \node[smalllabel,anchor=west] at (4.35,.50) {$\approx n\tau$ active tail records};
\end{scope}

\draw[arrow] (10.55,.78) -- (11.10,.78);

\begin{scope}[shift={(11.28,0)}]
  \draw[box,fill=softorange] (0,0) rectangle (4.35,4.28);
  \node[title] at (2.18,3.88) {Lower-bound transfer};
  \node[smalllabel] at (2.18,3.10) {$\rho_\tau(\ell)=\mathbb E_Q a$};
  \node[font=\sffamily\bfseries\scriptsize,align=center] at (2.18,2.30) {ordinary SCO\\inside CVaR learning};
  \node[smalllabel] at (2.18,1.42) {$m\asymp n\tau$};
  \draw[arrow] (1.35,1.10) -- (3.00,1.10);
  \node[smalllabel] at (2.18,.55) {$\widetilde\Omega\!\left(R_0\frac{\sqrt d}{\varepsilon n\tau}\right)$};
\end{scope}
\end{tikzpicture}
\end{adjustbox}
\caption{\textbf{Tail-sample embedding reduction.} Ordinary expected-risk learning with $m$ samples is embedded into CVaR learning by activating each record with probability $\tau$ and assigning inactive records zero loss. The embedded loss $\ell(w;(T,Y))=T a(w;Y)$ places the informative observations in the upper $\tau$-tail, so CVaR recovers the ordinary expected loss on the active distribution. Thus the effective-sample-size heuristic is exact on the canonical tail-embedded subproblem: any private CVaR learner on $n$ samples would imply a private SCO learner on $m\simeq n\tau$ samples.}
\label{fig:tail-sample-embedding}
\end{figure}
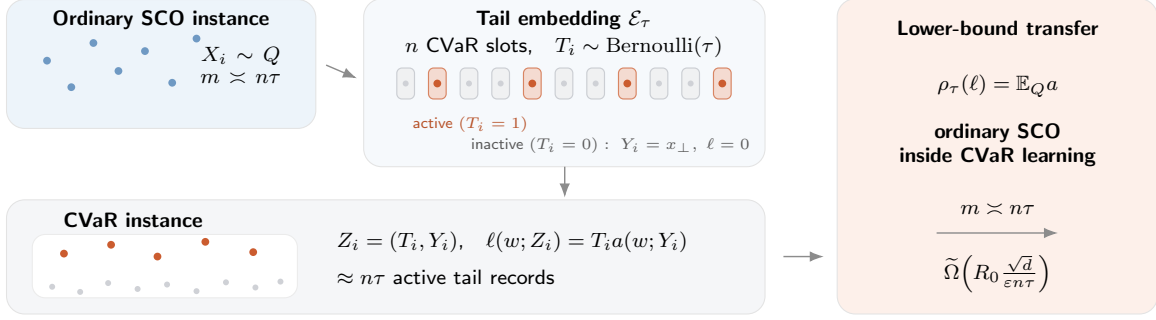

\subsection{Tail-sample transfer from CVaR learners to ordinary private SCO learners}

The embedding gives a formal way to transfer private SCO lower bounds to CVaR learning. On this embedded subclass, the effective-sample-size heuristic is exact: $n$ CVaR records contain only $\Theta(n\tau)$ ordinary informative records. For an ordinary loss class $\mathcal A$ with losses $a(w;x)\in[0,B]$, let $\mathsf{Emb}_\tau(\mathcal A)$ denote the CVaR class obtained by the dummy-inactive embedding above. For fixed privacy parameters $(\varepsilon,\delta)$, let $\mathfrak R^{\varepsilon,\delta,\SCO}_m(\mathcal A)$ denote the ordinary private SCO minimax excess risk with $m$ samples, and let $\mathfrak R^{\varepsilon,\delta,\CVaR}_{n,\tau}(\mathsf{Emb}_\tau(\mathcal A))$ denote the private CVaR minimax excess risk over the embedded class with $n$ samples. These are class-restricted transfer risks.

\begin{theorem}[Tail-sample transfer]
\label{thm:tail-sample-transfer}
Fix $n$ and $\tau$, and let $m=\lceil4n\tau\rceil$. Then
\[
\mathfrak R^{\varepsilon,\delta,\CVaR}_{n,\tau}(\mathsf{Emb}_\tau(\mathcal A))
\ge
\mathfrak R^{\varepsilon,\delta,\SCO}_m(\mathcal A)
-
B\exp(-c n\tau),
\]
for a universal constant $c>0$.
\end{theorem}

\begin{proof}
It suffices to prove the following algorithmic reduction. Suppose there is an $(\varepsilon,\delta)$-DP CVaR learner $\mathcal A_n$ that, for every embedded distribution $P_\tau(Q)$, achieves expected excess CVaR risk at most $\alpha_n$. For every $m\ge4n\tau$, an $(\varepsilon,\delta)$-DP ordinary learner is constructed using $m$ iid samples from $Q$ whose expected excess ordinary risk is at most
\[
\alpha_n+B\exp(-c m).
\]
Given ordinary samples $X_1,\ldots,X_m\sim Q$, construct a synthetic CVaR dataset of size $n$ as follows. Draw independent $T_1,\ldots,T_n\sim\operatorname{Bernoulli}(\tau)$ and let $K=\sum_iT_i$. If $K\le m$, assign the first $K$ ordinary samples to the active records $\{i:T_i=1\}$, and assign the dummy point $x_\bot$ to inactive records. If $K>m$, use the available $m$ ordinary samples for the first $m$ active records and fill the remaining active records with $x_\bot$.

Changing one input sample $X_j$ changes at most one synthetic record. Therefore the composed algorithm that constructs the synthetic dataset, runs $\mathcal A_n$, and outputs its predictor is $(\varepsilon,\delta)$-DP by post-processing \citep{DworkRoth2014} and the one-record sensitivity of the construction.

Couple this actual synthetic dataset $\widetilde S$ to an ideal dataset $S^\star\sim P_\tau(Q)^n$ by using the same Bernoulli indicators and the same active covariates whenever $K\le m$. Under this coupling, $\widetilde S=S^\star$ except on $\{K>m\}$, and therefore
\[
d_{\mathrm{TV}}\bigl(\mathcal L(\widetilde S),P_\tau(Q)^n\bigr)
\le
\Pp(K>m).
\]
The algorithmic output distribution after applying $\mathcal A_n$ is also within this total variation distance by post-processing. By Theorem 6.1, the ordinary excess risk of any output $w$ is exactly its embedded CVaR excess risk. Since this excess risk is bounded by $B$, the expected ordinary excess risk of the composed learner is at most
\[
\alpha_n+B\Pp(K>m).
\]
Finally, $\E K=n\tau$ and $m\ge4n\tau$, so a Chernoff bound \citep[Chapter~2]{BoucheronLugosiMassart2013} gives $\Pp(K>m)\le\exp(-cm)$ for a universal constant $c>0$. This proves the claimed ordinary-risk bound.

Taking the contrapositive and then $m=\lceil4n\tau\rceil$ gives the displayed minimax inequality, after adjusting the constant in the exponential residual.
\end{proof}

\begin{remark}
The reduction uses more ordinary samples than the expected number of active CVaR samples by only a constant factor. Thus any lower bound for private SCO with $m$ samples becomes a lower bound for CVaR learning with $n\tau\asymp m$ effective tail samples.
\end{remark}

\subsection{Convex lower bound and phase transition}

The convex lower bound uses only the lower interface in Theorem~\ref{thm:sco-interface}. The case $D=0$ is trivial, so assume $D>0$. To enforce the global loss cap $B$ without clipping, apply the Euclidean lower-bound primitive with Lipschitz parameter
\[
    G_0:=\min\{G,B/D\}.
\]
The resulting shifted linear-loss class is convex, $G_0$-Lipschitz, hence $G$-Lipschitz, and has range contained in an interval of length at most
\[
    R_0:=G_0D=\min\{B,GD\}.
\]
After shifting by a $w$-independent constant, the losses lie in $[0,R_0]\subseteq[0,B]$, so the CVaR tail embedding applies without changing excess risk. The logarithmic factors hidden below are exactly those inherited from the imported approximate-DP SCO lower interface, not from the CVaR embedding.

\begin{corollary}[Dimension-dependent private CVaR lower bound]
There exists a convex Lipschitz CVaR learning subclass over $W\subset\R^d$ with losses bounded by $B$ such that every central $(\varepsilon,\delta)$-DP learner with $\varepsilon\in(0,1]$ and $0<\delta\le\lceil4n\tau\rceil^{-2}$ satisfies, with $R_0=\min\{B,GD\}$,
\[
\sup_P
\E\left[
\rho_{\tau,P}(\widehat w)-\inf_{w\in W}\rho_{\tau,P}(w)
\right]
\ge
cR_0\min\left\{1,
\frac1{\sqrt{n\tau}}
+
\frac{\sqrt d}{\varepsilon n\tau}
\right\}
-
B\exp(-c'n\tau)
\]
up to the logarithmic factors inherited from the lower interface in Theorem~\ref{thm:sco-interface}, for universal constants $c,c'>0$.
\end{corollary}

\begin{proof}
Let $G_0=\min\{G,B/D\}$ and $R_0=G_0D=\min\{B,GD\}$. Use the shifted linear-loss class from the lower interface in Theorem~\ref{thm:sco-interface} with Lipschitz parameter $G_0$, diameter $D$, dimension $d$, and sample size $m=\lceil4n\tau\rceil$. The losses are convex, $G$-Lipschitz, and bounded in $[0,R_0]\subseteq[0,B]$.

Suppose a central $(\varepsilon,\delta)$-DP CVaR learner on $n$ samples achieved expected excess smaller than the displayed lower bound. Applying the tail-sample transfer theorem, Theorem~6.3, would produce a central $(\varepsilon,\delta)$-DP ordinary SCO learner with $m=\lceil4n\tau\rceil$ samples and expected excess below the Euclidean approximate-DP SCO lower interface, after accounting for the binomial-overflow residual. This contradicts the lower interface in Theorem~\ref{thm:sco-interface}. The privacy guarantee is preserved by Theorem~6.3 because changing one ordinary sample changes at most one synthetic active CVaR record.

The embedding in Theorem~6.1 preserves convexity and Lipschitzness because
\[
    \ell(w;(T,Y))=T a(w;Y), \qquad T\in\{0,1\}.
\]
It also preserves boundedness, since $0\le a(w;Y)\le R_0\le B$. Therefore the embedded class is a valid bounded convex Lipschitz CVaR subclass. Substituting $m=\lceil4n\tau\rceil$ into the lower interface gives
\[
    \widetilde\Omega\!\left(
        R_0
        \min\left\{
            1,\,
            \frac1{\sqrt{n\tau}}
            +
            \frac{\sqrt d}{\varepsilon n\tau}
        \right\}
    \right),
\]
up to the residual $B\exp(-c'n\tau)$ from Theorem~6.3.
\end{proof}

\paragraph{Residual term.}
The subtraction term is the binomial-overflow error in the coupling used by the tail-sample transfer theorem. The lower bound is therefore informative once $n\tau\gtrsim1$, and asymptotically once $n\tau\to\infty$. In the small-tail-sample regime $n\tau=O(1)$, the scalar and finite-class lower bounds already show the constant-scale privacy obstruction. The reduction also preserves record-level privacy exactly: changing one ordinary sample changes at most one synthetic active record before the CVaR learner is applied.

\paragraph{Convex comparison.}
The convex theory is modular: it separates CVaR-specific structure from the Euclidean approximate-DP SCO interface. The upper bound uses the rescaled Rockafellar--Uryasev lifting to reduce private CVaR optimization to a standard private SCO problem over $(w,u)$. The upper interface then gives the private term
\[
\widetilde O\!\left(
\frac{(GD+B)\sqrt{d\log(1/\delta)}}{\varepsilon n\tau}
\right).
\]
The lower bound runs in the opposite direction: the tail-embedding theorem shows that ordinary private SCO with $m$ samples is contained inside CVaR learning with $m\asymp n\tau$ informative tail samples. Combining this embedding with the lower interface gives
\[
\widetilde\Omega\!\left(
R_0
\min\left\{
1,\,
\frac1{\sqrt{n\tau}}
+
\frac{\sqrt d}{\varepsilon n\tau}
\right\}
\right),
\qquad
R_0=\min\{B,GD\}.
\]
Thus the privacy-dependent dependence on $\varepsilon,n,\tau$, and $d$ is matched up to the logarithmic factors inherited from Euclidean private SCO. The CVaR-specific reductions introduce the effective private tail sample size $\varepsilon n\tau$.
Figure~\ref{fig:convex-transfer-frontier} summarizes the two transfer directions and the matched private component.

\begin{samepage}
The corollary yields a clean impossibility threshold.

\begin{corollary}[Private tail-sample phase transition]
For Euclidean convex CVaR learning, nontrivial dimension-dependent private learning requires
\[
\varepsilon n\tau\gg\sqrt d
\]
up to logarithmic factors. More generally, the private effective tail sample size must dominate the relevant complexity of the hypothesis class.
\end{corollary}

\begin{proof}
Ignoring the exponentially small binomial-overflow residual from Theorem 6.3, as carried into the preceding corollary, the private lower-bound term is $R_0\sqrt d/(\varepsilon n\tau)$. If $\varepsilon n\tau\lesssim\sqrt d$, this term is of order $R_0$, the capped trivial scale of the bounded convex problem.
\end{proof}
\end{samepage}

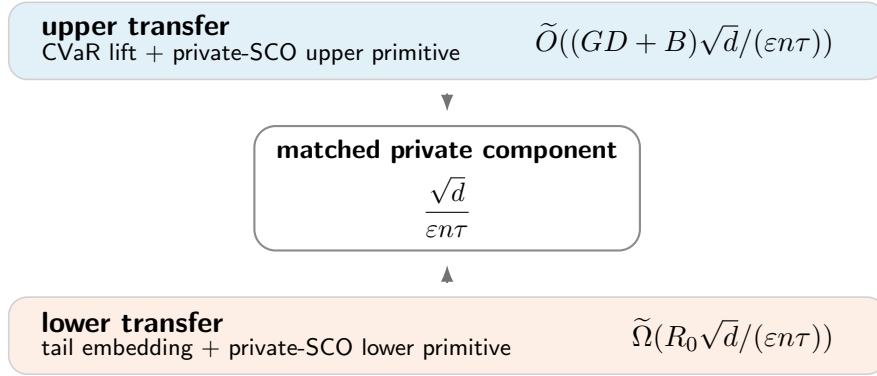
\begin{figure}[H]
\centering
\begin{adjustbox}{max width=\textwidth}
\begin{tikzpicture}[
  x=1cm,y=1cm,
  font=\sffamily\small,
  >=Latex,
  band/.style={rounded corners=7pt,draw=black!18,line width=.6pt,minimum width=11.6cm,minimum height=1.02cm,inner sep=0pt},
  middle/.style={rounded corners=8pt,draw=black!45,line width=.8pt,fill=white,minimum width=5.1cm,minimum height=1.45cm,align=center,inner sep=5pt},
  arrow/.style={->,line width=.75pt,draw=black!55,shorten >=5pt,shorten <=5pt}
]
\definecolor{upperblue}{RGB}{226,240,248}
\definecolor{lowerorange}{RGB}{253,239,229}

\node[band,fill=upperblue] (upper) at (0,1.95) {};
\node[anchor=west,font=\sffamily\bfseries] at (-5.50,2.13) {upper transfer};
\node[anchor=west,font=\sffamily\footnotesize] at (-5.50,1.77)
  {CVaR lift + private-SCO upper primitive};
\node[anchor=east,font=\normalsize] at (5.25,1.95)
  {$\widetilde O((GD+B)\sqrt d/(\varepsilon n\tau))$};

\node[middle] (mid) at (0,0) {
  {\bfseries matched private component}\\[4pt]
  $\displaystyle \frac{\sqrt d}{\varepsilon n\tau}$
};

\node[band,fill=lowerorange] (lower) at (0,-1.95) {};
\node[anchor=west,font=\sffamily\bfseries] at (-5.50,-1.77) {lower transfer};
\node[anchor=west,font=\sffamily\footnotesize] at (-5.50,-2.13)
  {tail embedding + private-SCO lower primitive};
\node[anchor=east,font=\normalsize] at (5.25,-1.95)
  {$\widetilde\Omega(R_0\sqrt d/(\varepsilon n\tau))$};

\draw[arrow] (upper.south) -- (mid.north);
\draw[arrow] (lower.north) -- (mid.south);
\end{tikzpicture}
\end{adjustbox}
\caption{\textbf{Matched convex upper and lower transfers.} The upper route applies the rescaled Rockafellar--Uryasev lift and imports the Euclidean approximate-DP SCO upper primitive, producing the displayed private term. The lower route uses the tail-sample embedding to transfer Euclidean private-SCO hardness into CVaR learning. The tilde-Omega notation suppresses the logarithmic factors inherited from the imported Euclidean private-SCO lower interface.}
\label{fig:convex-transfer-frontier}
\end{figure}

\begin{table}[!htbp]
\centering
\scriptsize
\setlength{\tabcolsep}{3pt}
\renewcommand{\arraystretch}{0.96}
\caption{Modular anatomy of the convex result.}
\begin{tabular}{p{0.24\linewidth}p{0.27\linewidth}p{0.39\linewidth}}
\toprule
Component & Source & Role \\
\midrule
Rockafellar--Uryasev lift &
Standard CVaR representation &
Converts CVaR minimization into a convex problem over $(w,u)$ with $\eta=\lambda u$. \\
Lifted Lipschitz scale &
Elementary calculation &
Gives scale $(GD+B)/\tau$, hence the private upper term after the SCO interface. \\
Private SCO upper interface &
Bassily--Feldman--Talwar--Thakurta &
Supplies the approximate-DP expected-population SCO upper bound. \\
Tail-sample embedding &
New reduction &
Embeds ordinary learning with $m\asymp n\tau$ informative samples into CVaR learning. \\
Private SCO lower interface &
Euclidean private-SCO lower theory &
Transfers ordinary private-SCO hardness to CVaR through the tail embedding. \\
\bottomrule
\end{tabular}
\end{table}

\section{A sensitivity extension to envelope-bounded coherent risks}

The exact minimax lower bounds are proved for CVaR. A simple sensitivity calculation, however, extends to coherent risks with bounded dual density envelopes in the standard coherent-risk framework \citep{ArtznerDelbaenEberHeath1999,RuszczynskiShapiro2006}.

Let
\[
\rho_{\mathcal Q}(X)=\sup_{q\in\mathcal Q}\E[qX],
\qquad
0\le q\le\kappa,
\qquad
\E q=1.
\]
The empirical analogue over a sample $S=(Z_1,\ldots,Z_n)$ assigns weights $q_i/n$ with $0\le q_i\le\kappa$ and $n^{-1}\sum_iq_i=1$.

\begin{proposition}[Sensitivity for envelope-bounded coherent risks]
If the realized loss values $x_i\in[0,B]$ and the dual envelope satisfies $q\le\kappa$ for every $q\in\mathcal Q$, then the empirical risk has one-record sensitivity at most
\[
B\min\left\{1,\frac{\kappa}{n}\right\}.
\]
For CVaR at tail mass $\tau$, $\kappa=1/\tau$.
\end{proposition}

\begin{proof}
Changing one sample changes any weighted empirical average by at most $(\kappa/n)B$. Taking a supremum over admissible weights cannot increase the sensitivity beyond the supremum of pointwise changes.
Since the risk of $[0,B]$-valued losses also lies in $[0,B]$, the one-record change is at most $B$. Combining the two bounds gives the display.
\end{proof}

Proposition 7.1 should be read as a sensitivity extension, not as a complete minimax theorem for arbitrary coherent risk envelopes. It shows that any envelope-bounded coherent risk with dual density radius $\kappa$ has empirical one-record sensitivity at most $B\min\{1,\kappa/n\}$. For CVaR, $\kappa=1/\tau$, and the preceding scalar, finite-class, and convex lower bounds show that this sensitivity scale is minimax sharp. Thus CVaR provides the canonical case in which the dual-envelope sensitivity mechanism is fully resolved.

\section{Discussion}

The results identify a privacy-tail-risk frontier by decomposing private CVaR excess into ordinary tail-risk statistics and a privacy price. The mechanism is visible in the empirical objective: CVaR gives each tail observation weight of order $1/(n\tau)$, and differential privacy must protect the contribution of any single such observation. The lower bounds show that this is not merely a defect of sensitivity-based algorithms. Scalar CVaR estimation closes the decomposition with rate
\[
\Theta\left(
B\min\left\{
1,
\frac1{\sqrt{n\tau}}
+
\frac1{\varepsilon n\tau}
\right\}
\right).
\]
Finite-class CVaR learning gives the anchor theorem, closing the same decomposition with rate
\[
\Theta\left(
B\min\left\{
1,
\sqrt{\frac{\log(2M)}{n\tau}}
+
\frac{\log(2M)}{\varepsilon n\tau}
\right\}
\right).
\]
The tail-sample transfer theorem further shows that ordinary private convex learning on $m$ examples is contained in CVaR learning with $m\asymp n\tau$ informative tail examples.

The results also clarify the relationship between CVaR learning and private quantile estimation. CVaR optimization contains a VaR threshold, but the lower bounds remain when the threshold is known. Therefore private quantile machinery can help with one component of the problem, but it cannot remove the private tail-average cost.

Several extensions are natural. First, the modular convex upper bound can be paired with sharper CVaR-specific empirical-process arguments, replacing the generic nonprivate term by distribution-dependent or tail-sample-size rates under additional regularity while preserving the same privacy price. Second, strongly convex losses should yield improved geometry-dependent private terms. Third, unbounded or heavy-tailed losses require private truncation or robust private tail estimation; without boundedness or moment/tail regularity, the scalar lower bounds suggest that uniform guarantees are impossible. Fourth, Proposition 7.1 suggests that envelope-bounded coherent risks may admit analogous privacy theories, with the dual density radius playing the role occupied by $1/\tau$ for CVaR.

\section{Conclusion}

Differential privacy imposes an unavoidable inverse-tail-mass penalty on CVaR learning. For tail mass $\tau$, the observations that determine the objective have effective sample size $n\tau$, and privacy further reduces the useful scale to $\varepsilon n\tau$. The results decompose private CVaR learning into an ordinary CVaR statistical term and a minimax-sharp privacy price. The latter is governed by the effective private tail sample size $\varepsilon n\tau$ and has the form
\[
\text{privacy term at scale }
\frac{\text{model complexity}}{\varepsilon n\tau},
\]
with the preceding lower bounds showing that the $1/\tau$ privacy price itself is intrinsic.
Thus CVaR is the canonical case in which the dual-envelope sensitivity $1/\tau$ is shown to be the minimax-sharp source of the privacy penalty.

\section*{Competing interests}
The author declares no competing interests.

\appendix

\section{A relative Bernstein bound for threshold-excess functions}
\label{app:relative-bernstein-thresholds}

This appendix proves the empirical-process step used in Lemma~\ref{lem:scalar-cvar-concentration}. The threshold-excess class is one-dimensional and nested, and its localized entropy is measured relative to its own $L_1(P)$ mass. This is the reason the concentration bound has no residual $\log(1/\tau)$ factor.

Let
\[
    H=\{h_\eta(x)=(x-\eta)_+:\eta\in[0,B]\}.
\]
All suprema below may be read as suprema over rational $\eta$'s. This causes no loss because $\eta\mapsto h_\eta$ is Lipschitz in the uniform norm. Thus the class is separable and standard measurability issues are avoided.

\begin{lemma}[Relative Bernstein for threshold excesses]
\label{lem:relative-bernstein-thresholds}
There is a universal constant $C>0$ such that, for every $u\ge1$, with probability at least $1-2e^{-u}$,
\[
    (P_n-P)h_\eta
    \le
    C\left(
        \sqrt{\frac{B\,P h_\eta\,u}{n}}
        +\frac{B u}{n}
    \right)
    \qquad\text{simultaneously for all }\eta\in[0,B].
\]
With probability at least $1-2e^{-u}$,
\[
    (P-P_n)h_\eta
    \le
    C\left(
        \sqrt{\frac{B\,P_n h_\eta\,u}{n}}
        +\frac{B u}{n}
    \right)
    \qquad\text{simultaneously for all }\eta\in[0,B].
\]
\end{lemma}

\begin{proof}
It is enough to prove the result for $B=1$. The general case follows by applying the normalized argument to $X/B$ and multiplying the resulting inequality by $B$. Hence assume $X\in[0,1]$.

Write
\[
    A(\eta)=P h_\eta=P(X-\eta)_+ .
\]
The map $A$ is continuous and nonincreasing on $[0,1]$, because
\[
    |h_\eta(x)-h_{\eta'}(x)|\le |\eta-\eta'|
    \qquad\text{for all }x\in[0,1].
\]
This continuity is the only fact needed to handle atoms and ties. No no-atom assumption is used.

For $r\in(0,1]$, define the localized class
\[
    H(r)=\{h_\eta\in H: P h_\eta\le r\}.
\]
Since $h_1\equiv0$, the class $H(r)$ is nonempty. Let
\[
    \eta_r=\inf\{\eta\in[0,1]:A(\eta)\le r\}.
\]
By continuity of $A$, $A(\eta_r)\le r$. Moreover, since $A$ is nonincreasing, $H(r)=\{h_\eta:\eta\ge \eta_r\}$ up to $P$-null distinctions. Thus $h_{\eta_r}$ is the largest element of $H(r)$ in the pointwise order. If $A$ is flat over an interval, any threshold in the flat part gives the same $L_1(P)$ bracket width, so ties cause no ambiguity.

Fix $\alpha\in(0,1)$, and put $a_r=A(\eta_r)\le r$. If $a_r=0$, then all functions in $H(r)$ are $P$-a.s. zero and one bracket of $L_1(P)$ width zero suffices. Otherwise, let
\[
    N=\left\lceil \frac{a_r}{\alpha r}\right\rceil\le \left\lceil\frac1\alpha\right\rceil
\]
and define levels
\[
    t_j=(a_r-j\alpha r)_+,\qquad j=0,\ldots,N .
\]
Using continuity and monotonicity of $A$, choose nondecreasing thresholds
\[
    \eta_r=\eta_0\le \eta_1\le\cdots\le \eta_N=1
\]
so that $A(\eta_j)=t_j$ whenever $t_j>0$, and $\eta_N=1$ when $t_N=0$. For every $j\ge1$, monotonicity gives
\[
    h_{\eta_j}\le h_\eta\le h_{\eta_{j-1}}
    \qquad
    \text{whenever }\eta\in[\eta_{j-1},\eta_j].
\]
The corresponding bracket width is
\[
    P(h_{\eta_{j-1}}-h_{\eta_j})
    =
    A(\eta_{j-1})-A(\eta_j)
    \le \alpha r .
\]
Therefore
\[
    N_{[]}\!\left(\alpha r,H(r),L_1(P)\right)
    \le
    \left\lceil\frac1\alpha\right\rceil+1,
\]
and hence
\[
    \log N_{[]}\!\left(\alpha r,H(r),L_1(P)\right)
    \le C_0\log(C_0/\alpha).
    \tag{A.1}
\]
Crucially, this bound depends on the relative resolution $\alpha$, but not on the radius $r$.

Since $0\le h\le1$, every $L_1(P)$ bracket of width $\delta$ is an $L_2(P)$ bracket of width at most $\sqrt{\delta}$. Thus, for $0<\varepsilon\le\sqrt r$,
\[
    \log N_{[]}\!\left(\varepsilon,H(r),L_2(P)\right)
    \le
    C_1\log\!\left(\frac{C_1 r}{\varepsilon^2}\right).
\]
Consequently,
\[
    \int_0^{\sqrt r}
    \sqrt{
        \log N_{[]}\!\left(\varepsilon,H(r),L_2(P)\right)
    }\,d\varepsilon
    \le
    C_2\sqrt r .
    \tag{A.2}
\]
Indeed, after the change of variables $\varepsilon=\sqrt r\,s$, the left side is bounded by
\[
    \sqrt r\int_0^1 \sqrt{\log(C_1/s^2)}\,ds,
\]
which is $O(\sqrt r)$.

The bracketing maximal inequality \citep[Theorem~2.14.2]{VanderVaartWellner1996} now gives, uniformly in $r\in(0,1]$,
\[
    \mathbb E
    \sup_{h\in H(r)}
    |(P_n-P)h|
    \le
    C_3\sqrt{\frac r n}.
    \tag{A.3}
\]
The following standard localized relative-deviation theorem is used in the exact form needed. It is the bounded, nonnegative, variance-controlled form of Bartlett, Bousquet and Mendelson's local Rademacher theorem; equivalently it follows from their Theorem~3.3 and Corollary~3.5 \citep{BartlettBousquetMendelson2005}.

\emph{Localized relative-deviation theorem.}
Let $\mathcal G$ be a countably separable class of functions satisfying $0\le g\le1$ and $\operatorname{Var}(g)\le Pg$. Suppose that for every $r\in(0,1]$,
\[
    \mathbb E\sup_{g\in\mathcal G:Pg\le r}|(P_n-P)g|
    \le \psi(r),
\]
where $\psi$ is a sub-root function with fixed point $r_*$. Then there is a universal constant $C$ such that, for every $u\ge1$, with probability at least $1-2e^{-u}$, simultaneously for all $g\in\mathcal G$ and all $K>1$,
\[
    P_n g
    \le
    \frac{K}{K-1}Pg
    +
    CK\left(r_*+\frac{u}{n}\right),
    \tag{A.4}
\]
and
\[
    Pg
    \le
    \frac{K}{K-1}P_n g
    +
    CK\left(r_*+\frac{u}{n}\right).
    \tag{A.5}
\]
Both inequalities are applied to the same nonnegative class $\mathcal G$; nothing is applied to $-\mathcal G$.

For the present class $H$, the variance condition holds because
\[
    \operatorname{Var}(h_\eta(X))
    \le P h_\eta^2
    \le P h_\eta,
\]
as $0\le h_\eta\le1$. By (A.3), the theorem applies with
\[
    \psi(r)=C_3\sqrt{\frac r n}.
\]
Its fixed point satisfies $r_*\le C_4/n$. Since $u\ge1$, put
\[
    t=r_*+\frac{u}{n}\le C_5\frac{u}{n}.
\]

From (A.4),
\[
    (P_n-P)h
    \le
    \frac{1}{K-1}Ph + CKt
    \qquad\text{for all }K>1.
\]
If $Ph\le t$, choose $K=2$ and get $(P_n-P)h\le C t$. If $Ph>t$, choose $K=1+\sqrt{Ph/t}$, which gives
\[
    (P_n-P)h
    \le
    C\sqrt{Ph\,t}+Ct.
\]
Therefore, simultaneously for all $h\in H$,
\[
    (P_n-P)h
    \le
    C\left(
        \sqrt{\frac{Ph\,u}{n}}
        +
        \frac{u}{n}
    \right).
    \tag{A.6}
\]

Similarly, (A.5) gives
\[
    (P-P_n)h
    \le
    \frac{1}{K-1}P_nh + CKt
    \qquad\text{for all }K>1.
\]
Optimizing in the same way, now using $P_nh$, yields
\[
    (P-P_n)h
    \le
    C\left(
        \sqrt{\frac{P_nh\,u}{n}}
        +
        \frac{u}{n}
    \right)
    \tag{A.7}
\]
simultaneously for all $h\in H$. This is the desired empirical-radius lower-tail bound. Notice that it was obtained from the two-sided relative deviation theorem for the nonnegative class $H$, not by applying an upper-tail theorem to $-H$.

Finally, rescale from $B=1$ to general $B$. If $Y=X/B$ and $\tilde h_{\eta/B}(Y)=(Y-\eta/B)_+$, then
\[
    h_\eta(X)=B\tilde h_{\eta/B}(Y),
    \qquad
    P h_\eta=B P\tilde h_{\eta/B}.
\]
Multiplying (A.6) and (A.7) by $B$ gives
\[
    (P_n-P)h_\eta
    \le
    C\left(
        \sqrt{\frac{B\,P h_\eta\,u}{n}}
        +
        \frac{Bu}{n}
    \right),
\]
and
\[
    (P-P_n)h_\eta
    \le
    C\left(
        \sqrt{\frac{B\,P_n h_\eta\,u}{n}}
        +
        \frac{Bu}{n}
    \right).
\]
This proves both claims.
\end{proof}

\section{Additional details on the exponential mechanism}

For completeness, recall the standard expected utility guarantee of the exponential mechanism \citep{McSherryTalwar2007} used in the finite-class upper bound. If a finite-range score $q(S,r)$ has sensitivity $\Delta_q$ and the exponential mechanism samples $r$ with probability proportional to $\exp(\varepsilon q(S,r)/(2\Delta_q))$, then
\[
\E\left[\max_r q(S,r)-q(S,\widehat r)\mid S\right]
\le
\frac{2\Delta_q}{\varepsilon}(\log M+1).
\]
Using $q=-\widehat\rho$ gives the displayed empirical excess bound.

\section{Constants in the scalar lower bound}

The proof of Theorem 3.2 is intentionally stated with universal constants. One admissible calibration is as follows. Let $p=c_1\min\{\tau,1/(\varepsilon n)\}$ and $k=\lceil4np\rceil$. Since $\varepsilon\le1$,
\[
k\varepsilon\le (4np+1)\varepsilon\le4c_1+1.
\]
Thus, for $c_1$ a small fixed numerical constant, $e^{-k\varepsilon}\ge e^{-(4c_1+1)}$. Since $P_0^n(E)\ge7/8$ and $P(K\le k)\ge3/4$, pure DP gives
\[
P_1^n(E)
\ge
\frac34 e^{-(4c_1+1)}\frac78
>
\frac18,
\]
contradicting $P_1^n(E)\le1/8$. In the approximate case, it suffices to require $k\delta\le1/8$ after adjusting constants.

\section{Why the tail-embedding distribution is a legitimate CVaR instance}

The embedded loss $\ell(w;(T,Y))=Ta(w;Y)$ has a large atom at zero because inactive records use the fixed dummy covariate $x_\bot$. This is not a pathology; it is exactly the mathematical structure of rare-event learning. The top $\tau$ tail of $\ell$ contains all active examples $T=1$ and possibly zero-loss inactive examples to fill the tail. Because inactive examples contribute zero loss and active losses are nonnegative, the CVaR equals the active conditional mean. This is why the embedding is exact and why the lower bound does not depend on approximation or asymptotics.

\end{document}